\documentclass{article} 
\usepackage[preprint]{colm2026_conference}

\usepackage{microtype}
\usepackage{graphicx}
\usepackage{subcaption}
\usepackage{booktabs} 
\usepackage{tabularx,array}
\newcolumntype{Y}{>{\centering\arraybackslash}X}
\usepackage[ruled,vlined]{algorithm2e}
\usepackage{multicol}
\usepackage{multirow}
\SetKwProg{Fn}{Function}{}{end}\SetKwFunction{FRecurs}{FnRecursive}%

\usepackage{xcolor}
\definecolor{xueyanblue}{HTML}{6893e9}
\usepackage{hyperref}
\hypersetup{
    colorlinks=true,
    linkcolor=xueyanblue,
    citecolor=xueyanblue,
    urlcolor=xueyanblue,
}

\usepackage{amsmath}
\usepackage{amssymb}
\usepackage{mathtools}
\usepackage{amsthm}
\usepackage{bm}
\usepackage{bbold}

\usepackage{amsmath,amsfonts,bm}

\def\eqref#1{equation~\ref{#1}}

\def\1{\bm{1}}

\def\rvs{{\mathbf{s}}}

\def\rvx{{\mathbf{x}}}
\def\rvy{{\mathbf{y}}}

\DeclareMathAlphabet{\mathsfit}{\encodingdefault}{\sfdefault}{m}{sl}
\SetMathAlphabet{\mathsfit}{bold}{\encodingdefault}{\sfdefault}{bx}{n}

\newcommand{\E}{\mathbb{E}}

\DeclareMathOperator*{\argmax}{arg\,max}
\DeclareMathOperator*{\argmin}{arg\,min}

\usepackage{enumitem}

\renewcommand{\paragraph}{\textbf}

\usepackage{tikz}

\usepackage{wrapfig}

\usepackage{pifont}  
\newcommand{\xmark}{\ding{55}}%

\usepackage{stfloats}

\usepackage{xurl}

\newcommand{\figpartlabel}[2]{%
  \phantomsection
  \label{#1-#2}
}
\newcommand{\figpartref}[2]{%
  \hyperref[#1-#2]{Figure~\ref*{#1}~{(#2)}}%
}

\definecolor{c0}{HTML}{CC6677}
\definecolor{c1}{HTML}{332288}
\definecolor{c2}{HTML}{DDCC77}
\definecolor{c3}{HTML}{117733}
\definecolor{c4}{HTML}{88CCEE}
\definecolor{c5}{HTML}{882255}
\definecolor{c6}{HTML}{44AA99}
\definecolor{c7}{HTML}{999933}
\definecolor{c8}{HTML}{AA4499}
\definecolor{c9}{HTML}{DDDDDD}

\usepackage[capitalize,noabbrev,nameinlink]{cleveref}
\crefname{enumi}{step}{steps}
\Crefname{enumi}{Step}{Steps}

\theoremstyle{plain}
\newtheorem{theorem}{Theorem}[section]
\newtheorem{proposition}[theorem]{Proposition}

\theoremstyle{definition}
\theoremstyle{plain}  
\newtheorem{remark}[theorem]{Remark}

\crefname{theorem}{Theorem}{Theorems}
\Crefname{theorem}{Theorem}{Theorems}

\crefname{proposition}{Proposition}{Propositions}
\Crefname{proposition}{Proposition}{Propositions}

\crefname{lemma}{Lemma}{Lemmas}
\Crefname{lemma}{Lemma}{Lemmas}

\crefname{corollary}{Corollary}{Corollaries}
\Crefname{corollary}{Corollary}{Corollaries}

\crefname{definition}{Definition}{Definitions}
\Crefname{definition}{Definition}{Definitions}

\crefname{assumption}{Assumption}{Assumptions}
\Crefname{assumption}{Assumption}{Assumptions}

\crefname{remark}{Remark}{Remarks}
\Crefname{remark}{Remark}{Remarks}

\newcounter{definition}[section]
\renewcommand{\thedefinition}{\thesection.\arabic{definition}}

\usepackage[textsize=tiny]{todonotes}

\usepackage{lineno}

\title{Efficient Test-Time Inference via Deterministic Exploration of Truncated Decoding Trees}

\author{Xueyan Li\textsuperscript{1,2}\qquad
Johannes Zenn\textsuperscript{1,3,4,6}\qquad
Ekaterina Fadeeva\textsuperscript{2}\qquad
Guinan Su\textsuperscript{1,3}\\[0.1em]
\textbf{Mrinmaya Sachan\textsuperscript{2}\qquad
Jonas Geiping\textsuperscript{1,3,5}}
\\[0.5em]
\textsuperscript{\textbf{1}}Max Planck Institute for Intelligent Systems \quad
\textsuperscript{\textbf{2}}ETH Zurich \quad
\textsuperscript{\textbf{3}}AI Center Tübingen \\
\textsuperscript{\textbf{4}}University of Tübingen \quad
\textsuperscript{\textbf{5}}ELLIS Institute Tübingen \quad
\textsuperscript{\textbf{6}}IMPRS-IS
}
\begin{document}

\ifcolmsubmission
\linenumbers
\fi

\maketitle
\lhead{}
\chead{}
\rhead{}
\renewcommand{\headrulewidth}{0pt}

\begin{abstract}
Self-consistency boosts inference-time performance by sampling multiple reasoning traces in parallel and voting. 
However, in constrained domains like math and code, this strategy is compute-inefficient because it samples with replacement, repeatedly revisiting the same high-probability prefixes and duplicate completions. We propose Distinct Leaf Enumeration (DLE), a deterministic decoding method that treats truncated sampling as traversal of a pruned decoding tree and systematically enumerates distinct leaves instead of sampling with replacement. This strategy improves inference efficiency in two ways. Algorithmically, it increases coverage of the truncated search space under a fixed budget by exploring previously unvisited high-probability branches. Systemically, it reuses shared prefixes and reduces redundant token generation. Empirically, DLE explores higher-quality reasoning traces than stochastic self-consistency, yielding better performance on math, coding, and general reasoning tasks.
\end{abstract}

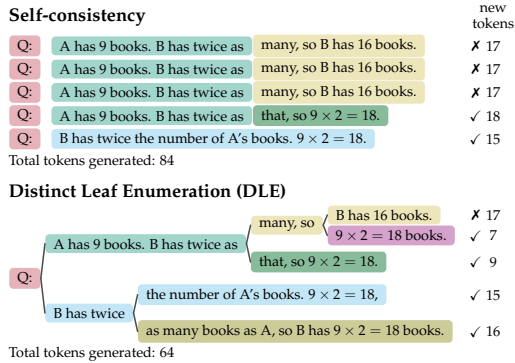
\begin{wrapfigure}{r}{0.5\linewidth}
    \figpartlabel{fig:1}{top}
    \figpartlabel{fig:1}{bottom}
    \vskip-2.25em
    \centering
    \resizebox{\linewidth}{!}{\begin{tikzpicture}[
  box/.style={rounded corners=2pt, inner xsep=4pt, inner ysep=2pt, align=left},
  conn/.style={line width=0.75pt, draw=black!70, line cap=round},
  font=\scriptsize,
  title/.style={font=\small\bfseries, anchor=west},
  subtitle/.style={font=\small\bfseries, anchor=west},
  wordcount/.style={font=\scriptsize, anchor=west}
]
\colorlet{cQ}{c0!50}
\colorlet{cOrange}{c6!50}
\colorlet{cOrange2}{c8!50}
\colorlet{cRed}{c2!50}
\colorlet{cGreen}{c3!50}
\colorlet{cCyan}{c4!50}
\colorlet{cBlue}{c7!50}

\node[title] (scTitle) at (-0.4,0) {Self-consistency};

\node[box, fill=cQ]      (q1)  at (0,-0.5) {Q:};
\node[box, fill=cOrange, right=0pt of q1.east, xshift=5pt, anchor=west] (r1a)
  {A has 9 books. B has twice as};
\node[box, fill=cRed, right=0pt of r1a] (r1b)
  {many, so B has 16 books.};

\node[box, fill=cQ]      (q2)  at (0,-0.9) {Q:};
\node[box, fill=cOrange, right=0pt of q2.east, xshift=5pt, anchor=west] (r2a)
  {A has 9 books. B has twice as};
\node[box, fill=cRed, right=0pt of r2a] (r2b)
  {many, so B has 16 books.};

\node[box, fill=cQ]      (q3)  at (0,-1.3) {Q:};
\node[box, fill=cOrange, right=0pt of q3.east, xshift=5pt, anchor=west] (r3a)
  {A has 9 books. B has twice as};
\node[box, fill=cRed, right=0pt of r3a] (r3b)
  {many, so B has 16 books.};

\node[box, fill=cQ]      (q4)  at (0,-1.7) {Q:};
\node[box, fill=cOrange, right=0pt of q4.east, xshift=5pt, anchor=west] (r4a)
  {A has 9 books. B has twice as};
\node[box, fill=cGreen, right=0pt of r4a] (r4b)
  {that, so $9\times 2=18$.};

\node[box, fill=cQ]      (q5)  at (0,-2.1) {Q:};
\node[box, fill=cCyan, right=0pt of q5.east, xshift=5pt, anchor=west] (r5a)
  {B has twice the number of A's books. $9\times 2=18$.};

\node[wordcount] (wordcount) at (-0.4,-2.5) {Total tokens generated: $84$};
\node[subtitle] (eleTitle) at (-0.4,-3.0) {Distinct Leaf Enumeration (DLE)};

\node[box, fill=cQ] (q0) at (0,-4.5) {Q:};

\node[box, fill=cOrange, right=0pt of q0.east, anchor=west, xshift=2pt, yshift=17pt] (aTop)
  {A has 9 books. B has twice as};
\node[box, fill=cCyan, right=0pt of q0.east, anchor=west, xshift=2pt, yshift=-17pt] (aBot)
  {B has twice};

\draw[conn] (q0.east) -- (aTop.west);
\draw[conn] (q0.east) -- (aBot.west);

\node[box, fill=cRed,   anchor=west, right=0pt of aTop.east, xshift=2pt, yshift=9pt] (tMid) {many, so};
\node[box, fill=cGreen, anchor=west, right=0pt of aTop.east, xshift=2pt, yshift=-9pt] (tAlt) {that, so $9\times 2=18$.};

\draw[conn] (aTop.east) -- (tMid.west);
\draw[conn] (aTop.east) -- (tAlt.west);

\node[box, fill=cRed,    anchor=west, right=0pt of tMid.east, xshift=2pt, yshift=5pt] (leafWrong) {B has 16 books.};
\node[box, fill=cOrange2, anchor=west, right=0pt of tMid.east, xshift=2pt, yshift=-5pt] (leafCalc) {$9\times 2=18$ books.};

\draw[conn] (tMid.east) -- (leafWrong.west);
\draw[conn] (tMid.east) -- (leafCalc.west);

\node[box, fill=cCyan, anchor=west, right=0pt of aBot.east, xshift=2pt, yshift=9pt] (bMid)
  {the number of A's books. $9\times 2=18$,};
\node[box, fill=cBlue, anchor=west, right=0pt of aBot.east, xshift=2pt, yshift=-9pt] (bLeaf)
  {as many books as A, so B has $9\times 2=18$ books.};

\draw[conn] (aBot.east) -- (bMid.west);
\draw[conn] (aBot.east) -- (bLeaf.west);

\coordinate (labelx) at ([xshift=6mm]leafWrong.east);
\node[box] (l1) at (labelx |- r1b.east) {\xmark};
\node[box] (l2) at (labelx |- r2b.east) {\xmark};
\node[box] (l3) at (labelx |- r3b.east) {\xmark};
\node[box] (l4) at (labelx |- r4b.east) {\checkmark};
\node[box] (l5) at (labelx |- r5a.east) {\checkmark};

\coordinate (labelxx) at ([xshift=9mm]leafWrong.east);
\node[box] (l12) at (labelxx |- r1b.east) {$17$};
\node[box] (l22) at (labelxx |- r2b.east) {$17$};
\node[box] (l32) at (labelxx |- r3b.east) {$17$};
\node[box] (l42) at (labelxx |- r4b.east) {$18$};
\node[box] (l52) at (labelxx |- r5a.east) {$15$};

\coordinate (labelx2) at ([xshift=6mm]leafWrong.east);
\node[box] (l1) at (labelx2 |- leafWrong.east) {\xmark};
\node[box] (l2) at (labelx2 |- leafCalc.east)  {\checkmark};
\node[box] (l3) at (labelx2 |- tAlt.east) {\checkmark};
\node[box] (l4) at (labelx2 |- bMid.east)  {\checkmark};
\node[box] (l5) at (labelx2 |- bLeaf.east)  {\checkmark};

\coordinate (labelx2) at ([xshift=9mm]leafWrong.east);
\node[box] (l11) at (labelx2 |- leafWrong.east) {$17$};
\node[box] (l21) at (labelx2 |- leafCalc.east)  {$7$};
\node[box] (l31) at (labelx2 |- tAlt.east) {$9$};
\node[box] (l41) at (labelx2 |- bMid.east)  {$15$};
\node[box] (l51) at (labelx2 |- bLeaf.east)  {$16$};

\node[wordcount] (wordcount) at (-0.4,-5.8) {Total tokens generated: $64$};

\node[box, align=center] (l1) at ([yshift=5mm] labelxx |- r1b.east) {new\\tokens};

\end{tikzpicture}}
    \caption{
    \textbf{Distinct Leaf Enumeration (DLE) avoids recomputing the same prefixes, explores the generation space systematically, and generates fewer new tokens.}}
    \label{fig:1}
    \vskip-1.5em
\end{wrapfigure}

\section{Introduction} \label{sec:introduction}

Parallel test-time computation has emerged as a principled way of scaling the performance of large language models (LLMs) \citep{snell_scaling_2024,lin_just_2024}. 
By sampling multiple solution traces for the same prompt and aggregating them via majority voting or answer-weighted selection \citep{wang_soft_2024,taubenfeld-etal-2025-confidence,wang_self-consistency_2023}, additional inference time compute can be used for performance gains. 
However, naive stochastic sampling of multiple solutions is inefficient. 
For constrained domains like math and coding where problems are only solved by a narrow set of viable solutions, similar reasoning steps are repeatedly generated \citep{zhu2024path,gao2025decoding,hong2025slim}. 

\looseness=-1
\figpartref{fig:1}{top} shows this redundancy where three out of five generations are the same (first three rows), and four out of five generations share a prefix. 
\Cref{fig:2} makes this more concrete by showing the repetition rate among prefixes: we find that for code generation, $20\%$ of all prefix tokens are repeated across generated traces. 
Thinking of generated sequences in terms of a tree structure, this means that early branches are repeatedly traversed. 
At the same time, the number of possible continuations grows exponentially per generated token. This makes search efficiency critical.

Classical search methods such as beam search \citep{lowerre1976harpy,sutskever2014sequence} make this tree structure explicit and control the traversed range with their number of branches (beams). 
However, beam search results in degenerate repetitions \citep{vijayakumar2016diverse,cohen_empirical_nodate}, quality deterioration with larger beams \citep{koehn_six_2017,yang_breaking_2018} and a lack of diversity \citep{su_contrastive_2022}.

\begin{wrapfigure}{r}{0.3\linewidth}
    \vskip-1.25em
    \centering
    \includegraphics{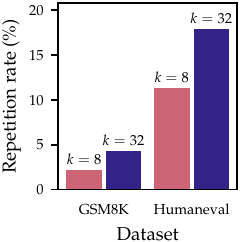}
    \caption{\textbf{The repetition rate of prefix tokens increases as more sequences are sampled for $\varepsilon$-sampling.}
    For code generation, the rate increases from $\approx11\%$ to $\approx17\%$.
    Details of calculation can be found in \Cref{appendix:datasets and parameters}.
    }
    \label{fig:2}
    \vskip-1.75em
\end{wrapfigure}

\looseness=-1
In this work, we argue that effective tree-based generation for modern LLMs requires two practical ingredients, a principled way to \textit{truncate} the exponentially large tree of possible completions, and a good heuristic for \textit{partial exploration} under a limited compute budget. 
Since the full decoding tree is intractable, the tree size can be reduced by a truncated next-token distribution. Common truncated samplers such as  top-$k$ \citep{fan_hierarchical_2018}, top-$p$ (nucleus sampling) \citep{holtzman_curious_2020}, min-$p$ \citep{minh2025turning} and $\varepsilon$-sampling \citep{hewitt_truncation_2022}, are widely used to stabilize generation with a restricted set of candidate tokens at each decoding step \citep{li_sample_2025}.
Unlike beam search, which considers alternatives at every position, truncation-based sampling branches only when multiple valid tokens survive the truncation rule.
Modern 
post-trained models are overconfident \citep{tian_just_2023,leng_taming_2025}, often exhibiting sharpened output distributions. 
Thus, only a few tokens survive truncation at each step which reduces the effective tree size. 

Motivated by these observations, we propose \textbf{Distinct Leaf Enumeration (DLE)}, an alternative to (sampled) self-consistency that minimizes repetitions via tree search. 
DLE \textit{enumerates} distinct leaves in the truncated generation tree by choosing \textit{deterministically} which tokens to explore instead of \textit{sampling} from the vocabulary probability distribution.
The output is a set of distinct reasoning chains that never generate duplicate traces. 
This makes DLE a natural replacement for self-consistency in settings where repeated sampling is especially wasteful, such as in code synthesis and tightly specified math problems.

\looseness=-1
DLE differs from prior decoding and search methods in both objective and mechanism. Unlike beam-style methods \citep{sutskever2014sequence,vijayakumar2016diverse}, it does not maintain diverse partial beams at each step. Unlike sampling-without-replacement methods \citep{kool_stochastic_2019,shi_incremental_2021}, it is not designed to preserve a target sampling distribution. Unlike tree-of-thought \citep{yao2023tree,xu2025phi} style branching, it requires no additional deliberation or evaluation calls. Instead, DLE deterministically enumerates distinct leaves of a truncated decoding tree to improve coverage under a fixed compute budget.

Further, DLE is able to \textit{reuse generated prefixes}.
Since it expands a pruned tree, it never re-generates the same prefix and lowers the number of newly generated tokens for the same number of sequences compared to self-consistency.

In this paper, we propose DLE, as a \textit{deterministic} enumeration method \textit{independent of the truncated sampling distribution}.
Empirically, at equal token budgets, DLE covers more probability space than self-consistency, translating to better performance.
This also results in reduced inference time especially in memory-constrained conditions.
Our main contributions are:

\begin{itemize}[leftmargin=*, itemsep=-0.em, topsep=-0.75em]
    \item We propose Distinct Leaf Enumeration (DLE), a repetition-free decoding method as a \emph{deterministic} drop-in replacement for stochastic sampling. It applies broadly to samplers such as top-$p$, min-$p$ and $\varepsilon$-sampling.
  
    \item 
    Algorithmic benefit: Empirically, we show that higher coverage of the truncated search space is associated with better downstream performance. DLE achieves higher coverage than stochastic self-consistency by systematically exploring previously unvisited high-probability branches, leading to stronger reasoning performance.
  
    \item 
    Systems benefit: We show empirically that DLE is more token-efficient than self-consistency. By avoiding redundant generations, DLE produces more complete sequences under a fixed token budget, which translates into practical latency improvements on modern inference systems, especially in memory-constrained settings.

\end{itemize}

\section{Related Work} \label{sec:related-work}

Our work targets a specific inefficiency in test-time scaling for constrained reasoning: repeated sampling produces many duplicate prefixes and often re-discovers the same high-probability traces. 
DLE addresses this by (i) enumerating \emph{distinct} leaves in a truncated sampling tree and (ii) exploiting prefix reuse across branches.

\paragraph{Test-time ensembling and self-consistency.}
Self-consistency aggregates multiple reasoning traces to improve accuracy, typically via majority voting on extracted answers \citep{wang_self-consistency_2023}. Follow-up work extends aggregation beyond simple answer extraction to free-form settings, e.g., by letting an LLM compare, synthesize, or select among sampled traces \citep{wang_soft_2024,chen2024universal,wang2024integrate}. 
Orthogonally, several methods reweight traces using confidence or selection criteria (e.g., self-certainty, filtering) to improve best-of-$N$ style inference \citep{kang_scalable_2025,fu_deep_2025}. 
In contrast, DLE focuses on making the \emph{sampling stage} itself non-redundant, guaranteeing a set of distinct answer traces.

\paragraph{Reducing redundancy in aggregation by pruning.}
Early-stopping self-consistency terminates sampling once answers stabilize \citep{li2024escape}, and adaptive-consistency allocates variable budgets across prompts and samples \citep{aggarwal2023let}. Path-based strategy reuses partial computation by expanding from promising prefixes, improving latency while maintaining accuracy \citep{zhu2024path}. 
DLE is complementary with a different mechanism: rather than adapting the number of i.i.d.\ samples, it deterministically enumerates \emph{previously unexplored} branches of the same truncated tree. 
This also connects to broader work on sampling sequences without replacement that aim to avoid duplicates while preserving distributional guarantees \citep{kool_stochastic_2019,shi_incremental_2021}.

\paragraph{Search-based decoding.}
Beam search \citep{lowerre1976harpy,sutskever2014sequence} makes the decoding tree explicit but suffers from degeneration and low diversity; diverse beam search addresses this by adding diversity-promoting terms \citep{vijayakumar2016diverse}. 
Other approaches promote beam diversity, e.g., via arithmetic coding style constructions \citep{vilnis2023arithmetic}. 
DLE differs from per-step breadth-first search. 
It traverses the pruned tree in a depth-first-like manner (greedy completions) and allocates a leaf budget to distinct trajectories, which is particularly well-matched to self-consistency-style voting in constrained domains.

\paragraph{Tree-based structures.}
Tree-of-thought style inference explores multiple reasoning paths with self-evaluation \citep{yao2023tree}. 
However, it requires repeated model prompting for each branch. 
Entroduction uses entropy and its variance to balance computation and exploration in multi-step reasoning \citep{zhang-etal-2025-entropy}, and $\phi$-decoding uses simulated foresight paths for globally informed step selection but also requires expensive extra computation \citep{xu2025phi}. 
DLE is cheaper by using only next-token probabilities and spends compute primarily on full leaf completions rather than partial rollouts plus re-ranking.

\paragraph{Efficiency via prefix reuse and decoding systems.}
Efficient implementations for multi-branch decoding leverage shared prefixes and specialized kernels or KV-cache data structures \citep{yao_deft_2024,juravsky_hydragen_2024,zheng2024sglang}. Trie-based decoding for beam search similarly shares KV cache across beams with common prefixes \citep{chan2025efficient}. 
DLE directly relates to these system optimizations because its exploration is explicitly tree-structured, which reuses the longest available prefix.

\section{Distinct Leaf Enumeration for Truncated Sampling Distributions} \label{sec:method}
In this section, we first introduce truncated sampling distributions generally, then, we motivate coverage as a metric. 
Finally, we introduce our distinct leaf enumeration algorithm.

\subsection{Truncated Sampling Distributions}

Here, we introduce truncated sampling distributions.
Let $p_t(\cdot \mid \rvx_{<t})$ denote the next-token distribution of the model at step $t$ given a prefix $\rvx_{<t}$ over vocabulary $\mathcal V$. 
A truncation rule defines a so-called \textit{active} set $A(\rvx_{<t}) =\{v \in \mathcal V: C(p_t(\cdot\mid \rvx_{<t}), v)\}$ consisting of tokens that fulfill a criterion $C$. 
Then, we can define the resulting truncated distribution $q$ as follows
\begin{align}
    q(v\mid \rvx_{<t}) =
    \begin{cases}
        \mathbb{1}[v = g(\rvx_{<t})] 
        & \text{if } |A(\rvx_{<t})| \le 1,\\
        \dfrac{p_t(v\mid \rvx_{<t})}{Z(\rvx_{<t})}\mathbb{1}[v\in A(\rvx_{<t})] 
        & \text{if } |A(\rvx_{<t})|\ge 2,
    \end{cases}
    \label{eq:qeps}
\end{align}
\looseness=-1
where $g(\rvx_{<t})\coloneqq\arg\max_{u} p_t(u\mid \rvx_{<t})$,
$\mathbb{1}$ is the indicator function,
and the normalization constant is given by $Z(\rvx_{<t})=\sum_{v\in A(\rvx_{<t})} p_t(v\mid \rvx_{<t})$.
For a sequence $\rvx_{1:T}\in\mathcal{V}^T$, the probability under the truncated distribution is the product of the conditional probabilities of \Cref{eq:qeps},
    $Q(\rvx_{1:T}) \;=\; \prod_{t=1}^{T} q(\rvx_t \mid \rvx_{<t}).$
Typically, $C$ removes \textit{low probability tokens} from the sampling distribution which are found to correlate with large \textit{epistemic uncertainty} \citet{li_sample_2025}.

\paragraph{Top-$k$ sampling \citep{fan_hierarchical_2018}.}
Top-$k$ sampling retains the $k$ highest-probability tokens under $p_t(\cdot \mid x_{<t})$.
Let $\pi_{\rvx_{<t}}$ be a permutation of the vocabulary such that $p_t(\pi_{\rvx_{<t}}(1)\mid \rvx_{<t}) \geq \cdots \geq p_t(\pi_{\rvx_{<t}}(|\mathcal V|)\mid \rvx_{<t})$.
Then, $C(p_t(\cdot \mid \rvx_{<t}), v) = \mathbb{1}[v \in \{\pi_{x_{<t}}(1), \dots, \pi_{x_{<t}}(k)\}]$.

\looseness=-1
\paragraph{Top-$p$ sampling \citep{holtzman_curious_2020}.}
top-$p$ retains the smallest set of tokens whose cumulative mass exceeds a threshold $p\in(0,1]$. 
$C(p_t(\cdot \mid \rvx_{<t}), v)$ is as for top-$k$ sampling above, but $
k = \argmin_{l}
\sum_{j=1}^{l} p_t(\pi_{\rvx_{<t}}(j)\mid \rvx_{<t}) \ge p
$.

\paragraph{min-$p$ sampling \citep{minh2025turning}.}
min-$p$ sampling uses a \emph{relative} threshold which retains all tokens whose probability is at least a fraction $p_{\min}$ of the most likely token.
$
    C(p_t(\cdot, \rvx_{<t}), v)
    =
    \mathbb{1}\!\left[
        p_t(v\mid \rvx_{<t})
        \ge
        p_{\min}\max_{u\in\mathcal V} p_t(u\mid \rvx_{<t})
    \right],
    $
where $p_{\min}\in(0,1]$. 

\paragraph{$\varepsilon$-sampling \citep{hewitt_truncation_2022}.}
\looseness=-1
At each generation step, $\varepsilon$-sampling only keeps tokens whose probability exceeds a fixed threshold $\varepsilon$.
It instantiates $C(p_t(\cdot, \rvx_{<t}), v)=\mathbb{1}[p_t(v\mid \rvx_{<t}) > \varepsilon]$ which considers only tokens above the threshold $\varepsilon$ where $\varepsilon\in(0,1]$.

\begin{wrapfigure}{r}{0.5\linewidth}
    \vskip-3.75em
    \centering
    \includegraphics[width=\linewidth]{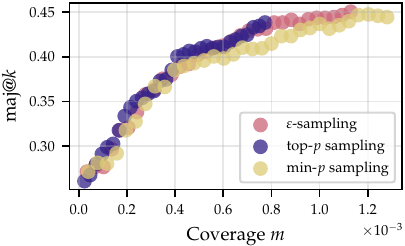}
    \caption{\textbf{
    More coverage of the decoding space by increasing the number of sequences aggregated ($k$) per question leads to better performance on GSM8K.
    }}
    \label{fig:coverage}
    \vskip-2.5em
\end{wrapfigure}
\subsection{Coverage} \label{sec:method:coverage}
Here, we define \textit{coverage} as the probability mass captured by a selected set of sequences under a truncated sequence distribution.

\refstepcounter{definition}
\label{def:coverage}
\textbf{Definition \thedefinition\ (Coverage).}
Let $\mathcal{L}'$ be any set of sequences sampled from a distribution $Q$ (which does not necessarily need to be a truncated distribution).
We define the \textit{coverage} $m_k$ of a set of sequences $\mathcal L'$ with $|\mathcal L'|=k$ as its probability mass under the truncated sampling distribution $Q$ as follows,
\begin{align}
m_k(\mathcal L') 
= \sum_{\rvx \in \mathcal L'} Q(\rvx).
\label{eq:coverage}
\end{align}
\vskip-0.25em
\looseness=-1
Empirically, coverage is indicative of performance.
In \Cref{fig:coverage} we find that coverage is positively correlated with performance across $\varepsilon$-sampling, top-$p$ sampling, and min-$p$ sampling.
We provide more details in \Cref{sec:exp:diversity}.
Next, we remark two properties of coverage.
\begin{remark}[Maximizing coverage over sequences]
\looseness=-1
Let $\mathcal L$ denote the set of all sequences that have non-zero probability mass.
Then, the maximizer of the coverage, i.e., $\argmax_{\mathcal{L}'} m_k(\mathcal L')$ is the set of the $k$ \textit{distinct} highest-probability sequences out of $\mathcal L$.
\end{remark}
%
\begin{remark}[Diminishing returns of expected coverage]
\looseness=-1
\textit{Expected} coverage $\E[m_k(\mathcal L)]$ (expectation over $k$ sequences of $Q$) leads to diminishing returns with increasing $k$ due to sampling duplicate sequences.
Specifically, the benefit of adding a new sequence $k+1$ is decreasing.
See \Cref{app:sec:expected-coverage-eps-sampling}.
\end{remark}

\subsection{Distinct Leaf Enumeration (DLE)} \label{sec:ele}

DLE is our proposed tree-based decoding algorithm that operates on the tree defined by a truncated sampling distribution (see above).
Specifically, the nodes in the tree $\rvx_{<t}$ represent prefixes with active sets $|A(\rvx_{<t})| > 1$ and the edges represent the next-token distribution $q(\cdot \mid\rvx_{<t})$.
We will refer to this tree as the \textit{pruned tree}.
DLE replaces sampling from $q$ by a \textit{deterministic} expansion scheme at a chosen node. It follows these steps:
\begin{enumerate}[leftmargin=*, itemsep=0pt, topsep=0pt, ref=\arabic*]
    \item 
    Given a prompt, DLE produces a greedy generation $\rvx_{1:T}$ until an end-of-sequence token (EOS) is reached, while keeping track of active sets along the path.
    
	\item 
    \label{enum:step:branch}
    DLE decides on the next branching location $i$ of the tree and picks an alternative token $v$ from the active set at that position.

	\item \label{enum:step:branch-greedy-generate}
    The alternative token is appended to the prefix $\rvx_{1:i-1}\circ v$.
    Then, DLE produces a greedy generation from that point until an EOS token is reached.
\end{enumerate}
\Cref{enum:step:branch} and \Cref{enum:step:branch-greedy-generate} are repeated until there either are no remaining active sets with non-trivial size $>1$ or until a maximum of \textit{$k$ leaf nodes} have been created.
The output of DLE is a set of distinct leaves $\mathcal{L}'$ and the corresponding weights $Q(\rvx)$ for $\rvx \in \mathcal{L}'$.

\looseness=-1
\paragraph{Expanding a node.}
A node $\rvx_{<t}$ is expanded depending on the size of its active set $A(\rvx_{<t})$.
\begin{itemize}[leftmargin=*, itemsep=0pt, topsep=0pt]
    \item 
    \textbf{Branching case $|A(\rvx_{<t})| > 1$}:
    One child is created for each token in the active set, which has prefix $\rvx_{<t}\circ v$, path probability $Q(\rvx_{<t} \circ v)$, and edge weight $q(v\mid\rvx_{<t})$.

    \item 
    \textbf{Non-branching case $|A(\rvx_{<t})|\le 1$}: 
    DLE creates one child from the largest token $v = g(\rvx_{<t})$ with prefix $\rvx_{<t}\circ v$, path probability $Q(\rvx_{<t}\circ v)$, and edge weight $1.0$.
\end{itemize}
\looseness=-1 For any prefix $\rvx_{<t}$ on the $\varepsilon$-pruned tree, its path probability is the product of the edge weights on its unique path from root to leaf.

\looseness=-1
\begin{remark}[If the number of leaves, $k$, is large, the support of the truncated sampling algorithm and the corresponding DLE is the same]
    \label[remark]{remark:support-equivalence-eps-sampling-ele}
    Note that if we explore all branches of the pruned tree, we enumerate all possible sequences $\rvx$ of the corresponding truncated sampling distribution with $Q(\rvx)>0$.
    Thus, their support is the same.
\end{remark}
\vskip-0.5em
\looseness=-1
While \Cref{remark:support-equivalence-eps-sampling-ele} makes a statement about arbitrary $k$, it can be practically infeasible to explore the full pruned tree, depending on its size.
Therefore, DLE deterministically expands previously unexplored branches of the same pruned tree to avoid generating duplicate leaves.
Using coverage, as defined in \Cref{def:coverage}, we see that, for the same leaf budget $k$, DLE allocates more compute towards uncovering new probability mass.
Therefore, we can conceptualize DLE as a \textit{compute-budgeted} procedure to order the exploration of the tree in a way that maximizes coverage (see \Cref{app:sec:greedy-sampling-optimality}).


DLE expands the branch that carries the largest probability mass seen so far. 
Each branch with prefix $\rvx_{<t}\circ v$ is scored as $Q(\rvx_{<t}\circ v)$ and we choose to expand the branch with the largest score.
This rule prioritizes exploring branches which currently seem most likely under $Q$ and, at the same time, is a greedy heuristic for uncovering new probability mass.
We provide an example of this algorithm to help understanding, and ablation study on various other branching rules in \cref{app:other algos}.


\begin{wrapfigure}{r}{0.5\linewidth}
    \vskip-1em
    \centering
    \resizebox{\linewidth}{!}{\begin{tikzpicture}[
  box/.style={rounded corners=2pt, inner xsep=4pt, inner ysep=2pt, align=left},
  conn/.style={line width=0.75pt, draw=black!70, line cap=round},
  font=\scriptsize,
  title/.style={font=\small\bfseries, anchor=west},
  subtitle/.style={font=\small\bfseries, anchor=west},
  wordcount/.style={font=\scriptsize, anchor=west}
]
\colorlet{cQ}{c0!50}
\colorlet{cRed}{c6!50}
\colorlet{cOrange2}{c8!50}
\colorlet{cBlue}{c2!50}
\colorlet{cGreen}{c3!50}
\colorlet{cOrange}{c4!50}
\colorlet{cCyan}{c7!50}


\node[box, fill=cQ] (q0) at (0,-4.15) {Q:};
\node[box, fill=cOrange, right=0pt of q0.east, anchor=west] (q01)
  {Firstly, [\dots]};

\node[box, fill=cOrange, right=0pt of q01.east, anchor=west, xshift=2pt, yshift=17pt] (aTop)
  {A has $9$ books. B has twice as};
\node[box, fill=cOrange, right=0pt of q01.east, anchor=west, xshift=2pt, yshift=-17pt] (aBot2)
  {B};
\node[box, fill=cCyan, right=0pt of aBot2.east, anchor=west] (aBot)
  {has twice the number of A's books.};

\draw[conn] (q01.east) -- (aTop.west);
\draw[conn] (q01.east) -- (aBot2.west);

\node[box, fill=cOrange,   anchor=west, right=0pt of aTop.east, xshift=2pt, yshift=9pt] (tMid) {many \underline{so B has} $16$ books.};
\node[box, fill=cOrange, anchor=west, right=0pt of aTop.east, xshift=2pt, yshift=-9pt] (tAlt2) {that\vphantom{\underline{B}}};
\node[box, fill=cGreen, anchor=west, right=0pt of tAlt2.east] (tAlt) {\underline{so B has}};

\node[box, right=0pt of tAlt.east, anchor=west, xshift=2pt, draw, line width=0.5pt, inner sep=2pt] (aBoti)
  {\textbf{STOP}};

\draw[conn] (aTop.east) -- (tMid.west);
\draw[conn] (aTop.east) -- (tAlt2.west);

\end{tikzpicture}}
    \caption{
    \textbf{Early stopping for DLE.}
    Schematic uses $n=3$ repetitions as early stopping condition.
    DLE generates the \textcolor{c4}{upper branch} first.
    When the \textcolor{c3}{middle branch} is generated, early stopping halts generation due to the tokens ``so B has'' being repeated.
    See \Cref{sec:ele}.
    }
    \vskip-1em
    \label{fig:early-stopping-branch}
\end{wrapfigure}
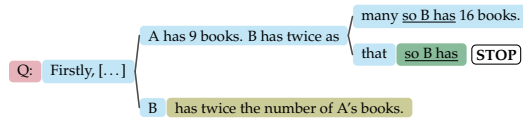
\paragraph{Early stopping.}
Branch alternatives can differ at the branching token but still lead to the \emph{same greedy suffix} later (e.g., semantically similar tokens converging under greedy decoding) \citep{hong2025slim}. 
To avoid repeatedly generating the same continuation, we early-stop \emph{after} the branch point once such a merge is detected (\Cref{fig:early-stopping-branch}).
If a branch matches $n=10$ consecutive tokens from a previously generated sibling continuation, we stop decoding that branch and reallocate the remaining sequence budget elsewhere. 
\Cref{app:sec:early-stopping} shows that this rule helps reduce the frequency of identical final answers since over 85\% of pruned branches would have produced the same final answer.

\subsection{Implementation} \label{sec:implementation}

In this section, we discuss leaf aggregation rules and the practical implementation of DLE.

\begin{table*}[!t]
\centering
\setlength{\tabcolsep}{3pt}
\renewcommand{\arraystretch}{1.2}
\newcommand{\na}{\multicolumn{1}{c}{\textemdash}}
{\small
\begin{tabularx}{\textwidth}{lYYYYYYYYY}
\toprule
\multirow{2}{*}{Method}
& \multicolumn{3}{c}{GSM8K} & \multicolumn{3}{c}{Humaneval} & \multicolumn{3}{c}{MMLU-Pro} \\
\cmidrule(lr){2-4} \cmidrule(lr){5-7} \cmidrule(lr){8-10}
& {\small maj@2} & {\small maj@4} & {\small maj@8}
& {\small pass@2} & {\small pass@4} & {\small pass@8}
& {\small maj@2} & {\small maj@4} & {\small maj@8} \\
\midrule
Self-consistency ($\tau=1$) & 19.26 & 25.32 & 31.61 & 18.29 & 26.22 & 33.54 & 13.61 & 15.15 & 16.24 \\
Self-consistency ($\tau=0.6$) & 29.57 & 36.24 & 41.32 & 31.10 & 42.07 & 48.17 & 16.15 & 16.43 & 17.84 \\
Beam search            & 35.10 & 36.47 & 36.92 & 31.10 & 29.88 & 32.93 & 15.49 & 15.29 & 15.01   \\
Diverse beam search    & - & 36.47 & 41.39 & - & 32.32 & 33.54 & - & 16.20 & 16.58 \\
DeepConf      & 21.68 & 26.91 & 33.36 & - & - & - & 14.41 & 15.58 & 16.76  \\
Self-certainty      & 20.09  & 25.32 & 32.22 & - & - & - & 14.51 & 15.92 & 14.68    \\
\midrule
Self-consistency (top-$p$\&top-$k$) & 25.47 & 31.84 & 38.74 & 31.71 & 36.59 & 45.73 & 14.96 & 16.66 & 17.35 \\
DLE (top-$p$\&top-$k$) & 34.80 & 41.17 & 44.43 & 39.02 & 43.90 & 51.22 & 16.36 & 17.30 & 18.19 \\
\midrule
Self-consistency (min-$p$) & 26.46 & 33.51 & 39.04 & 32.93 & 38.41 & 46.34 & 14.99 & 16.36 & 17.26 \\
DLE (min-$p$) & 33.74 & 39.58 & 43.97 & 38.41 & 45.12 & 53.05 & 16.45 & 17.40 & 17.89 \\
\midrule
Self-consistency ($\varepsilon$-sampling) & 26.84 & 33.51 & 40.94 & 32.93 & 39.63 & 47.56 & 14.89 & 16.07 & 17.20 \\
DLE ($\varepsilon$-sampling)  & 34.57 & 40.64 & 44.05 & 38.41 & 45.12 & 52.44 & 16.64 & 17.37 & 17.97  \\
\bottomrule
\end{tabularx}
}
\caption{
\looseness=-1
\textbf{DLE with various samplers improves performance compared to stochastic self-consistency.}
Performance (maj@k and pass@k) of various methods on GSM8K, Humaneval and MMLU-Pro with Qwen2.5-0.5B-Instruct with parameters given in \Cref{sec:exp}. 
Values are left blank when the method is not applicable to the metric. .
}
\label{table:main}
\end{table*}

Each generated leaf $\rvx \in \mathcal{L}'$ comes with its truncated sampling probability $Q(\rvx)$, which can be used for weighted aggregation with self-consistency.
In contrast to \citet{wang_soft_2024}, we find that uniform weighting is the most robust choice (see ablation in \cref{app:answer aggregation methods})

We highlight that DLE explores new branches from previously generated prefixes, so shared prefixes do not need to be recomputed and their KV caches can be reused. In contrast, standard self-consistency generates each sequence independently, repeatedly recomputing shared prefixes. As a result, DLE can benefit directly from inference engines that support prefix caching.
vLLM \citep{kwon2023efficient} uses \textit{automatic prefix caching} (APC), reusing cached KV blocks from prior requests when needed.
This lets any new request skip computing shared prefixes if they have been cached before.
SGLang \citep{zheng2024sglang} is built around KV cache reuse across multiple generation calls that is explicitly optimized for tree-based generations, using RadixAttention, which stores KV caches in a radix tree keyed by token prefixes. 
We show experiments on the performance of vLLM and SGLang in \Cref{sec:exp:computation}.

\subsection{Evaluation Metrics} \label{sec:metrics}

In this section, we define the metrics used to evaluate our experimental results in \Cref{sec:exp}.

\paragraph{Coverage.}
We define coverage $m(\widehat{\mathcal L})\in[0, 1]$ in \Cref{def:coverage} over a set of leaves $\mathcal L'$ quantifying how much of the pruned tree or probability mass DLE has explored.

\looseness=-1
\paragraph{Cache hit rate.}
We measure cache reuse over the \emph{flattened tree} of the $k$ enumerated leaves. 
Each leaf corresponds to one generation stream, and we concatenate all streams to obtain a single measure $L_{\mathrm{flat}}$ that includes all repeated prompt and generated tokens. 
    The \textit{actual cache hit rate} is the total number of cached tokens reported by SGLang (via \texttt{cached\_tokens}) divided by the total flat length: ${C_{\mathrm{act}}}/{L_{\mathrm{flat}}}$.
    The \textit{theoretical cache hit rate} is the maximum number of tokens that \emph{could} be reused given the enumerated tree structure, i.e., the length of the longest prefix of a leaf that matches any previously generated leaf (including the prompt): ${C_{\mathrm{th}}}/{L_{\mathrm{flat}}}$.

\section{Experiments} \label{sec:exp}

\begin{figure}[t]
    \centering
    \includegraphics{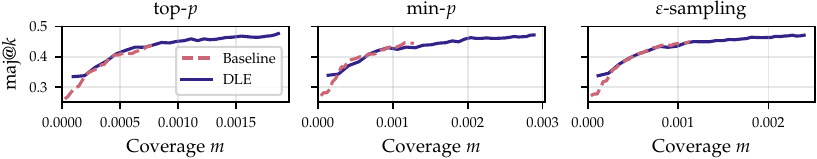}
    \caption{\textbf{DLE achieves higher coverage and higher maj@$k$ than sampled self-consistency.} 
    Across truncation rules, accuracy increases with coverage for both methods, but DLE covers more probability mass at the same sequence budget. See \Cref{sec:exp:diversity}.}
    \label{fig:coverage-accuracy-pairs}
\end{figure}

\begin{figure}[b]
    \centering
    \figpartlabel{fig:coverage-accuracy}{left}
    \figpartlabel{fig:coverage-accuracy}{right}
    \includegraphics{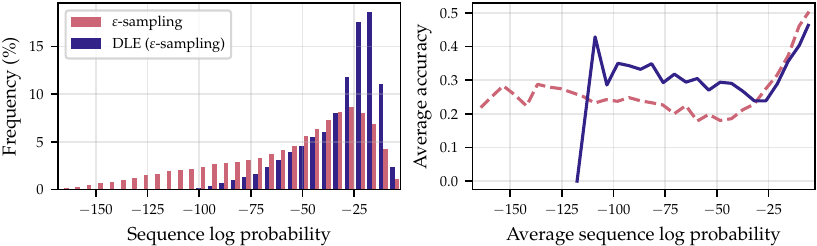}
    \caption{
    \looseness=-1
    \textbf{Left: The per sequence $\log$-probabilities are higher in DLE ($\varepsilon$-sampling) compared to self-consistency.}
    \textbf{Right: Sequences with log-probability $\geq-100$ are more accurate for DLE ($\varepsilon$-sampling) compared to self-consistency.}
    }
    \label{fig:coverage-accuracy}
\end{figure}

\looseness=-1
This section evaluates the proposed DLE algorithm on reasoning and coding tasks with various truncated sampling distributions (see \Cref{sec:method}).
We choose $\varepsilon=0.05$, $p=0.1$ for min-$p$, for top-$k$ we use $k=10$ in combination with top-$p$, $p=0.95$ which is commonly used in Qwen models \citep{qwen_qwen25_2025}. We run experiments with the Qwen \citep{qwen_qwen25_2025} and Llama \citep{dubey_llama_2024} model families.
We test mathematical reasoning with GSM8K \citep{cobbe_training_2021}, code generation with HumanEval \citep{chen2021evaluatinglargelanguagemodels}, and general reasoning with MMLU-Pro \citep{wang_mmlu-pro_2024}.
We compare DLE to other state-of-the-art parallel reasoning strategies such as DeepConf \citep{fu_deep_2025}, Self-Certainty \citep{kang_scalable_2025}, beam search \citep{lowerre1976harpy,sutskever2014sequence}, and diverse beam search \citep{vijayakumar2016diverse}.
For additional details, refer to \Cref{appendix:datasets and parameters} and \cref{app:additional results}.

\subsection{Coverage and Sequence Probability} \label{sec:exp:diversity}

\looseness=-1
\Cref{fig:coverage-accuracy-pairs} plots coverage against maj@$k$ across increasing values of $k$. 
For both DLE and the corresponding sampling baseline, accuracy improves as more probability mass is covered, indicating that the two methods follow a similar overall coverage-accuracy relationship. 
The key difference is that DLE moves further along this curve under the same sequence budget. 
For a given number of generated sequences, it attains higher coverage, which translates into higher performance as measured by maj@$k$. 

This interpretation is supported by \figpartref{fig:coverage-accuracy}{left}, where the per sequence $\log$-probabilities are higher for DLE than for the corresponding $\varepsilon$-sampling baseline. In \figpartref{fig:coverage-accuracy}{right}, these higher-probability sequences are also more accurate on average. Together, these results provide evidence that coverage is a meaningful metric for understanding performance.

\Cref{fig:coverage-epsilon} further shows that DLE increases coverage across sampler hyperparameters. The gain is largest when the per-step active set is larger. For example, in top-$p$ sampling, a larger $p$ includes more candidate tokens at each step, so the probability mass is distributed across more options. As a result, stochastic sampling is more likely to spend budget on lower-probability branches. DLE, in contrast, prioritizes branches with high sequence-level probability mass and then completes them greedily, allowing it to cover more of the search space under the same sequence budget.

\begin{figure}[t]
    \centering
    \includegraphics{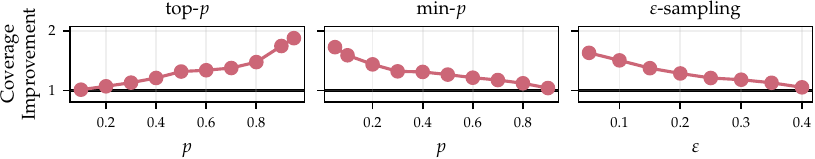}
    \caption{
    \textbf{DLE covers more of the search space compared to its baseline truncated sampling distributions.}
    Ratio of coverage $m$ (see \Cref{def:coverage}) between DLE and corresponding baseline plotted against $p$ (top-$p$, min-$p$) and $\varepsilon$ ($\varepsilon$-sampling).
    See \Cref{sec:exp:diversity}. 
    }
    \label{fig:coverage-epsilon}
\end{figure}

\begin{figure}[b] 
    \figpartlabel{fig:accuracy-number-tokens-eps-ele}{left}
    \figpartlabel{fig:accuracy-number-tokens-eps-ele}{right}
    \centering
    \includegraphics[width=0.495\textwidth]{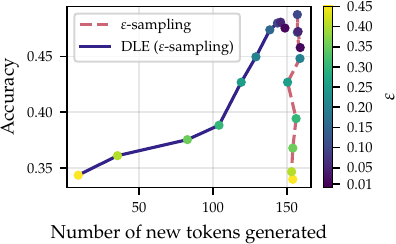}
    \includegraphics[width=0.495\textwidth]{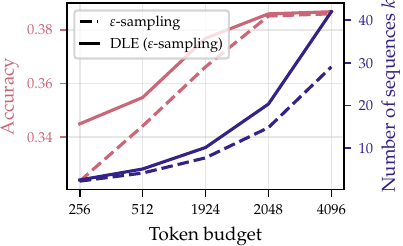}
    \caption{
    \textbf{Left}:
    \textbf{DLE ($\varepsilon$-sampling) outperforms $\varepsilon$-sampled self-consistency with far fewer new tokens generated.}
    The number of new tokens is averaged over $k=32$ for each question.
    $\varepsilon$-sampling generates many more tokens for similar performance as compared to DLE.
    $\varepsilon$ is colored according to its value.
    See \Cref{sec:exp:performance}.
    \textbf{Right}:
    \textbf{For a fixed token budget, DLE achieves a higher accuracy and generates more sequences compared to $\varepsilon$-sampling.}
    }
    \label{fig:accuracy-number-tokens-eps-ele}
\end{figure}

\subsection{Performance Improvements} \label{sec:exp:performance}

\looseness=-1
\Cref{table:main} shows various performance metrics for DLE with the three investigated samplers top-$p$ \& top-$k$, min-$p$, and $\varepsilon$-sampling for Qwen2.5-0.5B-Instruct. 
We use the same $k$ as beam size for the beam search method and also test the other methods with $k$ completed solutions.

We find that DLE improves upon sampling-based self-consistency with the corresponding truncated sampling baseline.
We observe the largest improvements for GSM8K ($\approx3$-$9$\%) followed by large improvements for Humaneval ($\approx4$-$7$\%). Beam search and diverse beam search do not yield consistent improvements as beam size and compute increase (as compared to the self-consistency baseline). 
This shows their limited effectiveness as a breadth first, per-step tree search algorithm.

\paragraph{DLE ($\varepsilon$-sampling) needs fewer tokens to reach the same accuracy compared to self-consistency.}
\figpartref{fig:accuracy-number-tokens-eps-ele}{left} shows that DLE ($\varepsilon$-sampling) requires far fewer tokens than $\varepsilon$-sampling to achieve a similar accuracy in both majority voting and pass@$k$.
To reach a performance of $\approx 34\%$, $\varepsilon$-sampled self-consistency would generate $>150$ tokens with $\varepsilon=0.45$ while DLE only needs $<20$ new tokens on average per question (for each of the 32 sequences) for similar performance.
Larger $\varepsilon$ results in more prefix reuse as branching happens later. 
This is also reflected in a higher theoretical cache reuse rate in \figpartref{fig:runtime-epsilon}{left}.
For small $\varepsilon$ values, the performance plateaus for DLE while it \textit{degrades} for $\varepsilon$-sampling.

\looseness=-1
\paragraph{DLE ($\varepsilon$-sampling) achieves higher accuracy than $\varepsilon$-sampling for a fixed token budget.}
Fixing the token budget for each question, DLE ($\varepsilon$-sampling) can complete more sequences by reusing prefixes, therefore producing a larger number of final answers for majority voting as shown in \figpartref{fig:accuracy-number-tokens-eps-ele}{right}. 
At low token budgets, this effect leads to significant gains for DLE. 

\subsection{Inference Speed} \label{sec:exp:computation}

\looseness=-1
In \Cref{sec:implementation}, we highlighted \textit{prefix reuse} as a major factor responsible for the efficiency of DLE.
\figpartref{fig:runtime-epsilon}{left} details the cache hit rate for DLE ($\varepsilon$-sampling) and self-consistency.
We find that DLE shows a higher cache hit rate because answer prefixes can be reused. 
In contrast, baseline $\varepsilon$-sampling can only reuse $k-1$ of all $k$ question prompts. 
All answer tokens have to be generated. 
Further, the actual cache hit rate (solid) is very close to the theoretical cache hit rate (dashed) as SGLang effectively retrieves the relevant KV cache from memory.
In \Cref{app:sec:computational-experiments}, we compare runtimes for vLLM and provide additional results.

In \figpartref{fig:runtime-epsilon}{right} we find that DLE ($\varepsilon$-sampling) achieves lower runtime on the SGLang inference engine.
This improvement is the most significant for small batch sizes of $1$ and $2$ where speed gains from parallelizing the generation of all traces cannot be effectively exploited by the baseline approach. 

\begin{figure}[!t]
    \centering
    \figpartlabel{fig:runtime-epsilon}{left}
    \figpartlabel{fig:runtime-epsilon}{right}
    \includegraphics{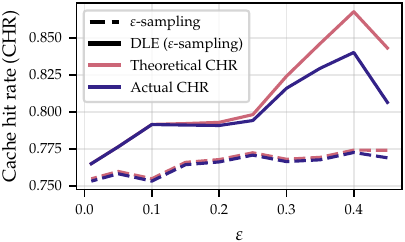}
    \includegraphics{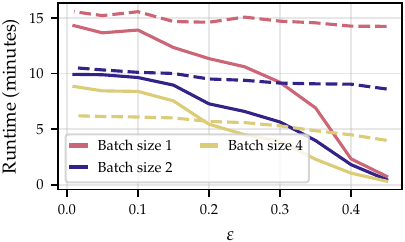}
    \caption{
    \textbf{Left}:
    \textbf{Cache rate is larger for DLE ($\varepsilon$-sampling) than $\varepsilon$-sampling.}
    SGLang ($k=8$) retrieves prefix cache efficiently such that the actual cache hit rate is close to the theoretical one. 
    See \Cref{sec:exp:computation}.
    \textbf{Right}:
    \textbf{Higher $\varepsilon$ results in much lower runtime for DLE ($\varepsilon$-sampling) than $\varepsilon$-sampling at $k=32$.}
    The time benefit is the most significant at small batch sizes and high $\varepsilon$. 
    Higher $\varepsilon$ enables more prefix reuse and larger runtime benefits.
    See \Cref{sec:exp:computation}.
    }
    \label{fig:runtime-epsilon}
\end{figure}

\looseness=-1
\paragraph{Algorithmic versus systems gains.}
DLE leads to benefits at two levels.
Algorithmically, it avoids duplicate leaf exploration and improves coverage of the truncated search space under a fixed budget. Systemically, its tree-structured exploration exposes shared prefixes that inference engines can exploit with KV-cache reuse. The latter gains depend on the serving stack, cache management, and batching strategy, and may vary across deployment settings. In practice, DLE is most attractive in constrained reasoning settings, smaller-budget regimes, and cache-aware inference engines, where duplicate exploration is common and prefix reuse can be exploited.

\section{Future Work and Conclusion}

\looseness=-1
\paragraph{Future work.}
One promising direction is to learn better branching policies that combine sequence probability with task-dependent quality signals. Another is to study adaptive hybrids of deterministic and stochastic exploration, for example by using DLE for high-mass regions and sampling for lower-mass tails. More broadly, the tree perspective may provide a useful interface between decoding and training, since branch points expose localized decisions that could support finer-grained credit assignment than sequence-level rewards.

\paragraph{Conclusion.}
We introduced Distinct Leaf Enumeration (DLE), a deterministic decoding strategy that explores truncated decoding trees without repeatedly sampling duplicate traces. Across reasoning and coding tasks, DLE improves coverage of the truncated search space and translates this into better performance and practical efficiency gains. More broadly, our results suggest that effective test-time scaling depends not only on generating more samples, but also on allocating computation more efficiently across plausible continuations.

\clearpage





\bibliography{colm2026_conference}
\bibliographystyle{colm2026_conference}

\newpage
\appendix
\onecolumn

\section{Datasets and Parameters}\label[appendix]{appendix:datasets and parameters}
MMLU-Pro \citep{wang_mmlu-pro_2024} metrics are weighted by the number of questions in each category.

All plotted diagrams and tables use Qwen2.5-0.5B-Instruct unless stated otherwise. All experiments are run on B200 or H100 GPUs using vLLM. The default benchmark for diagrams is GSM8K unless stated otherwise.

\cref{fig:2} measures repetition rate using only generated answer tokens, excluding the prompt. For each question, we compare $k$ generated answers and count repeated prefix tokens: a token is counted as repeated if it belongs to an initial prefix of length $m$ that matches the first $m$ tokens of any earlier generation. Per question repetition rate is averaged across all questions in each dataset. \cref{fig:2} uses $\varepsilon$-sampling with $\varepsilon=0.05$. Higher $\varepsilon$ results in higher repetition rate. 

\figpartref{fig:accuracy-number-tokens-eps-ele}{left} uses $k=32$ and DLE with $\varepsilon$-sampling.
The number of new tokens generated is summed over the full dataset, and averaged per question per $k$. Since at high $\varepsilon$, DLE does not need to use the full 32-sequences budget to explore the tree fully, its averaged number of new tokens is low. 

\figpartref{fig:accuracy-number-tokens-eps-ele}{right} uses DLE ($\varepsilon$-sampling, $\varepsilon=0.3$), and a fixed token budget for each question. For each question, sequences are generated one by one (not in parallel) until all tokens are used up.

\clearpage

\section{Ablations of the Proposed Algorithm} \label[appendix]{app:sec:early-stopping}

\subsection{Search Algorithms}\label[appendix]{app:other algos}

\Cref{fig:expansion-schemes} shows branching variations \textsc{ProbFirst} (main text) and \textsc{DivFirst} for DLE with $\varepsilon$-sampling truncated distribution.

\begin{figure}[h]
    \centering
    \resizebox{0.5\linewidth}{!}{\begin{tikzpicture}[
  box/.style={rounded corners=2pt, inner xsep=4pt, inner ysep=2pt, align=left},
  conn/.style={line width=0.75pt, draw=black!70, line cap=round},
  font=\scriptsize,
  title/.style={font=\small\bfseries, anchor=west},
  subtitle/.style={font=\small\bfseries, anchor=west}
]
\colorlet{cQ}{c0!50}
\colorlet{cOrange}{c6!50}
\colorlet{cRed}{c2!50}
\colorlet{cGreen}{c3!50}
\colorlet{cCyan}{c4!50}
\colorlet{cBlue}{c7!50}

\tikzset{
    tconn/.style={line width=0.75pt, draw=black!70, line cap=round, font=\tiny},
}


\node[box, fill=cQ] (q0) at (0,-4.15) {$Q_\varepsilon = 1$};

\node[box, fill=cQ, right=0pt of q0.east, anchor=west, xshift=15pt, yshift=17pt] (aTop)
  {$Q_\varepsilon = 0.9$};
\node[box, fill=cQ, right=0pt of q0.east, anchor=west, xshift=15pt, yshift=-17pt] (aBotl)
  {$\vphantom{Q}$};
\node[box, fill=cCyan, right=0pt of aBotl.east, anchor=west] (aBot)
  {$Q_\varepsilon=0.1$};

\draw[tconn] (q0.east) -- node[pos=.4, above, xshift=-4pt] {$0.9$} (aTop.west);
\draw[tconn] (q0.east) -- node[pos=.4, below, xshift=-4pt] {$0.1$} (aBotl.west);

\node[box, fill=cQ,   anchor=west, right=0pt of aTop.east, xshift=15pt, yshift=9pt] (tMid) {$Q_\varepsilon = 0.63$};
\node[box, fill=cQ, anchor=west, right=0pt of aTop.east, xshift=15pt, yshift=-9pt] (tAltl) {$\vphantom{Q}$};

\node[box, fill=cGreen, anchor=west, right=0pt of tAltl.east] (tAlt) {$Q_\varepsilon = 0.27$};

\draw[tconn] (aTop.east) -- node[pos=.4, above, xshift=-2pt] {$0.7$}  (tMid.west);
\draw[tconn] (aTop.east) -- node[pos=.4, below, xshift=-2pt] {$0.3$}  (tAltl.west);

\node[box, fill=cQ,    anchor=west, right=0pt of tMid.east, xshift=15pt, yshift=5.5pt] (leafWrong) {$Q_\varepsilon = 0.504$};
\node[box, fill=cQ, anchor=west, right=0pt of tMid.east, xshift=15pt, yshift=-5.5pt] (leafCalcl) {$\vphantom{Q}$};
\node[box, fill=cOrange, anchor=west, right=0pt of leafCalcl.east] (leafCalc) {$Q_\varepsilon = 0.186$};

\draw[tconn] (tMid.east) -- node[pos=.4, above, xshift=-2pt] {$0.8$}  (leafWrong.west);
\draw[tconn] (tMid.east) -- node[pos=.4, below, xshift=-2pt] {$0.2$}  (leafCalcl.west);

\node[box, fill=cCyan, anchor=west, right=0pt of aBot.east, xshift=15pt, yshift=9pt] (bMid)
  {$Q_\varepsilon = 0.1$};
\node[box, fill=cCyan, anchor=west, right=0pt of aBot.east, xshift=15pt, yshift=-9pt] (bLeafl)
  {$\vphantom{Q}$};
\node[box, anchor=west, right=0pt of bLeafl.east] (bLeaf)
  {\textcolor{red}{\Large\xmark}};

\draw[tconn] (aBot.east) -- node[pos=.4, above, xshift=-2pt] {$0.95$} (bMid.west);
\draw[tconn] (aBot.east) -- node[pos=.4, below, xshift=-2pt] {$0.05$} (bLeafl.west);

\coordinate (labelx) at ([xshift=6mm]leafWrong.east);
\node[box] (l1) at (labelx |- leafWrong.east) {$L_1$};
\node[box] (l2) at (labelx |- leafCalc.east)  {$L_2$};
\node[box] (l3) at (labelx |- tAlt.east) {$L_3$};
\node[box] (l4) at (labelx |- bMid.east)  {$L_4$};
\node[box] (l5) at (labelx |- bLeaf.east)  {$L_5$};

\end{tikzpicture}}
    \caption{
    \textbf{Branching variations of DLE ($\varepsilon$-sampling).}
    Toy example with $\varepsilon=0.1$.
    \textsc{ProbFirst}: most likely branches are explored as \textcolor{c0}{$L_1$}, \textcolor{c3}{$L_3$}, \textcolor{c6}{$L_2$}, \textcolor{c4}{$L_4$}.
    \textsc{DivFirst}: earliest branches are explored as \textcolor{c0}{$L_1$}, \textcolor{c4}{$L_4$}, \textcolor{c3}{$L_3$}, \textcolor{c6}{$L_2$}.
    Branch \textcolor{c7}{$L_5$} falls below the $\varepsilon$ threshold and is not explored. 
    Colors distinguish sampling rollouts.
    }
    \label{fig:expansion-schemes}
\end{figure}
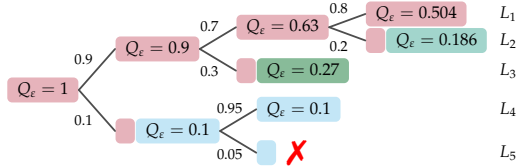

\paragraph{Branching with \textsc{ProbFirst}.}
\textsc{ProbFirst} expands the branch that carries the largest probability mass seen so far. 
Each branch with prefix $\rvx_{<t}\circ v$ is scored as $Q_\varepsilon(\rvx_{<t}\circ v)$ and we choose to expand the branch with the largest score.
This rule prioritizes exploring branches which currently seem most likely under $Q_\varepsilon$.
\Cref{fig:expansion-schemes} shows the generation of branches in order.
With $\varepsilon=0.1$, \textsc{ProbFirst} \textcolor{c0}{first generates branch $L_1$ } greedily after which 3 branch points will appear with path probabilities $0.1$, $0.27$ and $0.186$. It explores \textcolor{c3}{branch $L_3$} as it is the largest, then the next largest \textcolor{c6}{branch $L_2$}.
Then during the exploration for \textcolor{c4}{$L_4$}, two branches emerge but \textcolor{c7}{branch $L_5$} falls below the threshold so it is pruned. 
The remaining branch \textcolor{c4}{$L_4$} becomes deterministic under $\varepsilon$-sampling at that node (edge weight 1).
We can contrast this approach with a diversity-first strategy:

\paragraph{Branching with \textsc{DivFirst}.}
The goal of \textsc{DivFirst} is to encourage diversity \textit{early} in the generated sequences. 
The branching point is selected according to the earliest token position. 
This forces early divergence near the root.
In \Cref{fig:expansion-schemes}, \textsc{DivFirst} creates \textcolor{c0}{branch $L_1$} first greedily. Then chooses the top alternative token at the first branch point that creates \textcolor{c4}{$L_4$} during which \textcolor{c7}{$L_5$} is rejected. The next earliest branch points is at \textcolor{c3}{$L_3$}, and then at \textcolor{c6}{$L_2$}.

The branching algorithm described in the main paper is called \textsc{ProbFirst}. In addition to \textsc{ProbFirst}, we test four additional search algorithms
\begin{itemize}
    \item 
    \textsc{DivFirst}
    selects a branching point according to the earliest available alternative token position. 
    This forces early divergence near the root.

    \item \textsc{RandBranch} chooses branching points by \textit{sampling} from the distribution defined by the \textsc{ProbFirst} weights. We define a categorical distribution proportional to the probability mass under $\varepsilon$-sampling and sample it to get the next branching point. This is no longer deterministic, unlike \textsc{ProbFirst}.
    \item \textsc{GlobalProb} chooses branches that have the highest global edge weight. Instead of choosing branches that have the highest global path weight so far (from accumulated edge probabilities), it chooses to explore the alternative token that has the highest probability among all the branches explored so far.
    \item \textsc{DFS} is Depth-First-Search that starts with greedy generation, but deepest branches are explored first. This prioritizes late branching.
\end{itemize}

\Cref{table:early stopping} shows that \textsc{DivFirst} performs on par with \textsc{ProbFirst}. Other alternative methods are worse without early-stopping. 

\begin{table}[!h]
\centering
\setlength{\tabcolsep}{3pt}
\renewcommand{\arraystretch}{1.2}
\newcommand{\na}{\multicolumn{1}{c}{\textemdash}}
{\small
\begin{tabularx}{\textwidth}{lYYYYYY}
\toprule
\multirow{2}{*}{Method}
& \multicolumn{3}{c}{GSM8K} & \multicolumn{3}{c}{Humaneval}  \\
\cmidrule(lr){2-4} \cmidrule(lr){5-7}
& {\small maj@2} & {\small maj@4} & {\small maj@8}
& {\small pass@2} & {\small pass@4} & {\small pass@8} \\
\midrule
DLE-\textsc{DivFirst}     & 35.41 & 40.41 & 44.35& 39.02 & 46.34 & 51.22   \\
 w/o early stopping & 34.95 & 40.49 & 45.11 & 39.02 & 46.34 & 50.61 \\
\midrule
DLE-\textsc{ProbFirst}  & 34.57 & 40.64 & 44.05 & 38.41 & 45.12 & 52.44  \\
 w/o early stopping & 34.27 & 39.95 & 43.29 & 38.41 & 45.12 & 51.22 \\
\midrule
DLE-\textsc{RandBranch}  & 34.72 & 39.27 & 44.43 & 34.76 & 46.34 & 53.05    \\
 w/o early stopping & 33.97 & 39.04 & 42.68 & 35.98 & 42.68 & 48.78 \\
\midrule
DLE-\textsc{DFS} &33.81&33.81&34.12&31.10&31.71&31.71\\
\midrule
DLE-\textsc{GlobalProb} &33.97&34.72&35.25&34.76&37.8&40.24\\
\bottomrule
\end{tabularx}
}
\caption{
Adding an early-stopping condition to DLE ($\varepsilon$-sampling) baselines generally improves performance. Alternative branching algorithms \textsc{DFS} and \textsc{GlobalProb} yield worse performance.
}
\label{table:early stopping}
\end{table}

\subsection{Early Stopping}\label[appendix]{app:early stopping}
 
To avoid repeatedly generating the same continuation, we early-stop.
For example, suppose we branch at position $i$, producing two candidates
$\rvx_{1:i-1}\circ v$ and $\rvx_{1:i-1}\circ v'$ with $v\neq v'$.
We greedily decode forward $n$ steps from the new branch, obtaining a short suffix
$s = g_1 g_2 \cdots g_{n}$
where $g_j$ is the $j$-th greedy token after the branching token.
If this suffix $s$ matches the first $n$ greedy tokens after position $i$ of \emph{any previously generated leaf} that branched from the same prefix $\rvx_{1:i-1}$ (i.e., an earlier explored sibling branch), then we declare that the two rollouts have merged and terminate the current branch.

We provide an ablation that removes the early stopping condition on DLE with $\varepsilon$-sampling. \cref{table:early stopping} shows that adding early stopping generally results in better performance with low cost of additional tokens. 

\begin{figure*}[!h]
    \centering
    \includegraphics{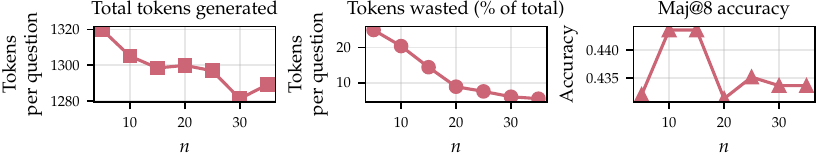}
    \caption{The $x$-axis shows $n$, the number of repeated token generated from a branch point, after which this branch will be terminated for DLE ($\varepsilon$-sampling). As $n$ increases, the early-stopping condition is triggered less frequently, resulting in fewer wasted tokens.
    }
    \label{fig:app:token-wasted}
\end{figure*}

\Cref{fig:app:token-wasted} shows that the fraction of wasted tokens remains small relative to the total number of generated tokens, so the heuristic adds little compute overhead. The best $n$ is around 10--15 and we use $n=10$ in the main paper.

\begin{table}[!b]
    \small
    \centering
    \begin{tabular}{lccc}
        \toprule
         & Trigger rate & Exact suffix match & Exact answer match \\
        \midrule
        k=8   & 8.93\% & 16.85\% & 85.87\% \\
        k=32   & 9.96\% & 19.96\% & 88.13\% \\
        \bottomrule
    \end{tabular}
    \caption{The early-stopped branches result in identical final answers with high probability, motivating the addition of early-stopping rule, with low cost of additional compute.}
    \label{table:ngram overlap rate}
\end{table}

To quantify whether early stopping prevents redundant answer exploration, we measure:
\begin{itemize}
    \item Trigger rate: how often branch-level n-gram repetition heuristic is triggered. How many branches are early stopped, out of all explored branches.
    \item Exact suffix match: how often the early-stopped branches would have continued to produce the exact full continuation, if generation had been allowed to continue with no early-stop rule. In other words, among the pruned branches, what percentage are true duplicate races rather than branches that only matched on the initial n-gram but later diverged.
    \item Exact answer match: how often the early-stopped branches would have produced the exact same final answer.
\end{itemize}

\cref{table:ngram overlap rate} shows that early-stopped sequences account for only a small fraction of explored branches, with trigger rates below 10\% in both settings. This indicates that the branch-level n-gram heuristic is relatively selective. It does not fire often, but instead targets a narrow subset of branches that are especially likely to revisit an already explored continuation.
Among the branches that are pruned, exact suffix match is much lower than exact answer match. This suggests that the heuristic is not primarily removing branches that are token-for-token identical all the way to the end. Many pruned branches still differ at the surface-form level after the branch point. However, despite these suffix-level differences, the vast majority of them still lead to the same final answer, with exact answer match above 85\% in both settings.
Taken together, these results suggest that the early-stopping rule is effective at avoiding redundant answer exploration rather than only eliminating fully identical traces. Even when two branches diverge in their detailed continuation, the repetition signal is often a strong indicator that they will converge to the same answer. This makes the heuristic attractive in practice. It reduces unnecessary decoding on branches that are unlikely to contribute new answer-level diversity, while consuming little additional compute for abandoned branches (\cref{fig:app:token-wasted}).

\subsection{Diversity and Coverage}

\newcommand{\distinctn}{\mathrm{Distinct\text{-}}n}
\newcommand{\ngrams}{n\mathrm{\text{-}grams}}
\paragraph{N-gram diversity.}
The n-gram diversity of a string $\rvs$ is defined as the fraction of unique $\ngrams$ among all $\ngrams$ in $\rvs$ \citep{li_diversity-promoting_2016}.
Specifically,
$
     \distinctn(\rvs)
     = {|\mathrm{unique}(\ngrams(\rvs))|}/{|\ngrams(\rvs)|}.
$

We use $n=10$ for $\distinctn$ and early stopping.
\figpartref{fig:diversity-pass-maj-eps-ele}{a} shows $\varepsilon$-sampling has much larger $10$-gram diversity for $\varepsilon<0.3$. But that does not lead to better maj@k performance (\figpartref{fig:diversity-pass-maj-eps-ele}{b} and \figpartref{fig:diversity-pass-maj-eps-ele}{c}).
Generally, DLE outperforms the baseline despite lower diversity. 
This might be because n-gram diversity is a surface-level proxy for variation and is not directly aligned with correctness, so increasing it can change the form of generations without necessarily increasing the probability mass on correct solutions.

\begin{figure*}[!h]
    \centering
    \figpartlabel{fig:diversity-pass-maj-eps-ele}{a}
    \figpartlabel{fig:diversity-pass-maj-eps-ele}{b}
    \figpartlabel{fig:diversity-pass-maj-eps-ele}{c}
    \includegraphics{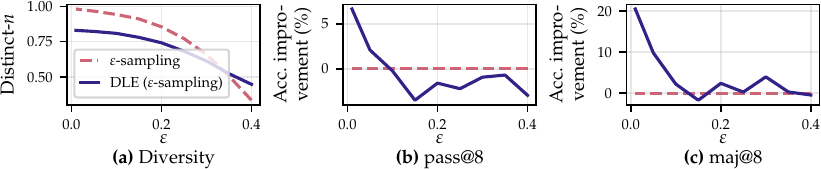}
    \caption{
    Diversity, pass@$8$ and maj@$8$ for DLE with $\varepsilon$-sampling.
    \textbf{(a)}
    $\varepsilon$-sampling shows the largest diversity for $\varepsilon<0.3$.
    \textbf{(b)} 
    Relative accuracy improvement (\%, pass@8) of DLE compared to $\varepsilon$-sampling is most significant at low $\varepsilon$. Pass@k reflects diversity, which does not necessarily lead to better maj@k performance as seen in
    \textbf{(c)}
    }
    \label{fig:diversity-pass-maj-eps-ele}
\end{figure*}

Below in \Cref{fig:coverage-epsilon-2}, we provide the raw data for \Cref{fig:coverage-epsilon}.
We draw the same conclusions as in the main text, see \Cref{sec:exp:diversity}. We emphasize that coverage should be viewed as a useful heuristic rather than a universal proxy for correctness. In constrained domains under truncated decoding, covering more of the retained probability mass is often associated with better performance because correct solutions tend to lie among a relatively small set of plausible continuations. However, this relationship need not hold in all settings, especially when correct answers lie in lower-probability regions or when model probabilities are poorly aligned with correctness.

\begin{figure}[h]
    \centering
    \includegraphics{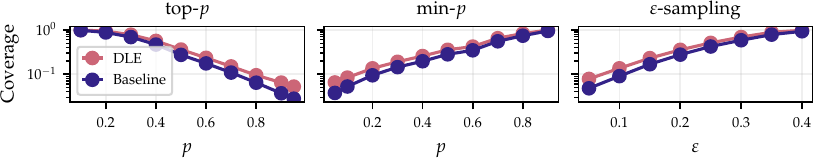}
    \caption{
    \textbf{DLE increases coverage over its baseline truncated sampling distributions.}
    Coverage $m$ (see \Cref{def:coverage}) between DLE and corresponding baseline plotted against $p$ (top-$p$, min-$p$) and $\varepsilon$ ($\varepsilon$-sampling).
    See \Cref{sec:exp:diversity} for more information.
    }
    \label{fig:coverage-epsilon-2}
\end{figure}

\subsection{Answer Aggregation Methods}\label[appendix]{app:answer aggregation methods}
Since we can obtain per-sequence probabilities during generation, we can aggregate answers by weighting them according to sequence-wise probabilities. However, we find that this method does not result in systematic improvement in performance, compared to weighting all final answers equally, as seen in \cref{table:answer-aggregation}.

\begin{table}[!h]
    \small
    \centering
    \begin{tabular}{lcc}
        \toprule
         & Probability weighted & Equally weighted \\
        \midrule
        \multicolumn{3}{l}{\textbf{Qwen2.5-0.5B-Instruct}} \\
        \textsc{DivFirst}   & 0.4170 & 0.4435 \\
        \textsc{ProbFirst}  & 0.4177 & 0.4405 \\
        \textsc{RandBranch} & 0.4109 & 0.4443 \\
        \midrule
        \multicolumn{3}{l}{\textbf{Qwen2.5-7B-Instruct}} \\
        \textsc{DivFirst}   & 0.8878 & 0.8992 \\
        \textsc{ProbFirst}  & 0.8916 & 0.8886 \\
        \textsc{RandBranch} & 0.8908 & 0.8893 \\
        \bottomrule
    \end{tabular}
    \caption{Weighting final answers by their leaf probabilities does not systematically improve performance. Decoding methods use DLE ($\varepsilon$-sampling)}
    \label{table:answer-aggregation}
\end{table}

\clearpage

\section{More Context on Distinct Leaf Enumeration}

\Cref{app:sec:expected-coverage-eps-sampling} shows that the expected coverage of DLE leads to diminishing returns when increasing $k$, while the motivation of DLE is to explore more \textit{new} probability mass.
\Cref{app:sec:greedy-sampling-optimality} shows that our greedy choice is locally optimal and minimizes computational cost.
\Cref{app:sec:acc-number-tokens-generated} shows an extension to \figpartref{fig:accuracy-number-tokens-eps-ele}{left}.

\subsection{Expected Coverage} \label[appendix]{app:sec:expected-coverage-eps-sampling}

Let $\mathcal L$ denote the set of all leaves in the pruned tree.
Let $\rvx^{(1)},\ldots,\rvx^{(k)} \sim Q$ be the $k$ leaves produced by the truncated sampling algorithm, and let
$S_k \coloneqq \{\rvx^{(1)},\ldots,\rvx^{(k)}\}$ be the set of unique leaves.

\begin{proposition}[Expected coverage of truncated sampling] \label{prop:expected-coverage-eps-sampling}
    The expected coverage of sampling the truncated distribution after $k$ samples is
    \begin{align}
    \mathbb{E}\left[m(S_k)\right]
    = \sum_{\rvx\in\mathcal L} Q(\rvx)\left(1-(1-Q(\rvx))^k\right).
    \label{eq:expected-coverage-eps-sampling}
    \end{align}
    The expected marginal gain is
    \begin{align}
    \mathbb{E}\left[m(S_{k+1})-m(S_k)\right]
    = \sum_{\rvx\in\mathcal L} Q(\rvx)^2 (1-Q(\rvx))^k,
    \label{eq:marginal-gain-eps-sampling}
    \end{align}
    which is non-increasing in $k$.
\end{proposition}

\begin{proof}
Let $\rvy^{(j)}$ for $j=1, \dots, k$ be the distinct leaves $\in S_k$.
Then,
\begin{align}
    \mathbb{E} \left[m(S_k)\right]
    &= \mathbb{E}\left[\sum_{\rvy\in S_k} Q(\rvy)\right] \\
    &= \mathbb{E}\left[\sum_{\rvy\in\mathcal L} Q(\rvy)\mathbb{1}\{\rvy\in S_k\}\right] \\
    &= \sum_{y\in\mathcal L} Q(\rvy)\mathbb{E}\left[\mathbb{1}\{\rvy\in S_k\}\right] \\
    &= \sum_{y\in\mathcal L} Q(\rvy)\mathbb{P}(\rvy\in S_k) \\
    &= \sum_{y\in\mathcal L} Q(\rvy)(1-\mathbb{P}(\rvy\notin S_k)) \\
    &= \sum_{y\in\mathcal L} Q(\rvy)\left(1-\mathbb{P}(\forall j\in\{1,\ldots,k\}: \rvy^{(j)} \neq \rvy)\right) \\
    &= \sum_{y\in\mathcal L} Q(\rvy)\left(1-\prod_{j=1}^k \mathbb{P}(Y_j\neq y)\right) \\
    &= \sum_{y\in\mathcal L} Q(\rvy)\left(1-(1-Q(y))^k\right).
\end{align}
The marginal gain follows by the difference of two of the terms.
\end{proof}

Therefore, while $\mathbb{E}\left[m(S_k)\right]$ is increasing in $k$, $\mathbb{E}\left[m(S_{k+1})-m(S_k)\right]$ is decreasing. 
This shows diminishing returns for the expected coverage of sampling from a truncated distribution.
Intuitively, this shows that the ``waste'' of truncated sampling comes from duplicate leaves since truncated sampling samples full leaves \textit{i.i.d.~with replacement}. 

\paragraph{Connection to DLE.}
In contrast, DLE deterministically expands previously unexplored branches of the same pruned tree which \textit{largely avoids generating duplicate leaves}.
Thus, for the same leaf budget $k$, DLE can allocate more compute towards \textit{uncovering new probability mass}.

\subsection{Greedy Sampling is Locally Optimal and Compute Optimal} \label[appendix]{app:sec:greedy-sampling-optimality}

After picking an alternative token at a branch point, we always complete the sequence greedily.
For a prefix $\rvx$ and a horizon $T$, its best continuation is $V^\ast(\rvx, T) = \max_{\rvx'_{1:T}\in \mathcal V^T}\ \log p(\rvx'_{1:T}\mid \rvx)$, with Bellman optimality recursion \citep{bellman1966dynamic}
\begin{align}
    \!\!\!
    V^\ast(\rvx, T) 
    =\! \max_{v\in A_\varepsilon(\rvx)}
    [
    \log p (v\mid \rvx) + V^\ast(\rvx\circ v, T-1)
    ],
\end{align}
where $V^\ast(\rvx,0)=0$.
Greedy sampling replaces the term $V^\ast(\rvx\circ v, T-1)$ by $0$.
\looseness=-1
\begin{remark}[Local and compute optimality of greedy sampling]
    Greedy sampling $v = g(\rvx_{<t})$ chooses, by definition, the maximum next-token probability $p(v\mid\rvx_{<t})\geq p(v' \mid\rvx_{<t})\forall v' \in \mathcal V$, and it generates a sample in $T$ forward passes.
\end{remark}



\clearpage
\section{Additional Results}\label[appendix]{app:additional results}

\subsection{Results for Other Models}

\Cref{table:qwen7b} shows results for Qwen2.5-7B-Instruct, Qwen2.5-14B-Instruct, Llama3.2-1B-Instruct, and Llama3.2-3B-Instruct with various samplers and branching algorithms.

\begin{table*}[!h]
\centering
\label{table:qwen7b}
\setlength{\tabcolsep}{3pt}
\renewcommand{\arraystretch}{1.2}
\newcommand{\na}{\multicolumn{1}{c}{\textemdash}}
{\small
\begin{tabularx}{\textwidth}{lYYYYYYYYY}
\toprule
\multirow{2}{*}{Method}
& \multicolumn{3}{c}{GSM8K} & \multicolumn{3}{c}{Humaneval} & \multicolumn{3}{c}{MMLU-Pro} \\
\cmidrule(lr){2-4} \cmidrule(lr){5-7} \cmidrule(lr){8-10}
& {\small maj@2} & {\small maj@4} & {\small maj@8}
& {\small pass@2} & {\small pass@4} & {\small pass@8}
& {\small maj@2} & {\small maj@4} & {\small maj@8} \\
\midrule
Self-consistency & 78.77 & 86.05 & 89.16 & 64.02 & 77.44 & 85.98 & 45.79 & 48.43 & 51.75 \\
\midrule
Self-consistency (top-p+top-k) & 81.12 & 87.79 & 89.23 & 68.90 & 78.05 & 87.20 & 45.89 & 47.11 & 50.13 \\
DLE (top-p+top-k) & 84.23 & 87.79 & 89.69 & 81.10 & 85.37 & 87.20 & 46.78 & 48.61 & 50.52 \\
\midrule
Self-consistency (min-p) & 81.58 & 87.49 & 89.84 & 70.43 & 80.38 & 83.40 & 46.14 & 48.09 & 49.88 \\
DLE (min-p) & 84.61 & 88.48 & 89.16 & 79.88 & 84.15 & 86.59 & 46.84 & 48.55 & 50.40 \\
\midrule
Self-consistency ($\varepsilon$-sampling) & 81.43 & 87.87 & 89.92 & 72.56 & 81.10 & 88.41 & 46.19 & 48.31 & 51.89 \\
DLE ($\varepsilon$-sampling)-\textsc{DivFirst}  & 84.91 & 88.48 & 89.92 & 81.10 & 86.59 & 87.80 & 46.89 & 48.90 & 50.71 \\
DLE ($\varepsilon$-sampling)-\textsc{ProbFirst} & 84.15 & 87.95 & 88.86 & 79.88 & 84.76 & 87.20 & 46.88 & 48.72 & 50.62 \\
DLE ($\varepsilon$-sampling)-\textsc{RandBranch} & 83.78 & 87.11 & 88.78 & 76.83 & 83.54 & 88.41 & 46.70 & 48.72 & 50.69 \\
\bottomrule
\end{tabularx}
}
\caption{
Performance (maj@k and pass@k) of various methods on GSM8K, Humaneval and MMLU-Pro with the Qwen2.5-7B-Instruct model.
}
\end{table*}

\begin{table*}[!h]
\centering
\label{table:qwen14b}
\setlength{\tabcolsep}{3pt}
\renewcommand{\arraystretch}{1.2}
\newcommand{\na}{\multicolumn{1}{c}{\textemdash}}
{\small
\begin{tabularx}{\textwidth}{lYYYYYYYYY}
\toprule
\multirow{2}{*}{Method}
& \multicolumn{3}{c}{GSM8K} & \multicolumn{3}{c}{Humaneval} & \multicolumn{3}{c}{MMLU-Pro} \\
\cmidrule(lr){2-4} \cmidrule(lr){5-7} \cmidrule(lr){8-10}
& {\small maj@2} & {\small maj@4} & {\small maj@8}
& {\small pass@2} & {\small pass@4} & {\small pass@8}
& {\small maj@2} & {\small maj@4} & {\small maj@8} \\
\midrule
Self-consistency & 88.48 & 91.51 & 92.49 & 67.68 & 78.05 & 88.41 & 52.01 & 55.50 & 58.39 \\
\midrule
Self-consistency (top-p+top-k) & 89.61 & 91.21 & 92.80 & 72.56 & 84.76 & 89.02 & 52.68 & 55.31 & 57.74 \\
DLE (top-p+top-k) & 89.92 & 91.81 & 91.74 & 76.83 & 84.15 & 88.41 & 54.50 & 56.27 & 58.06 \\
\midrule
Self-consistency (min-p) & 89.99 & 91.21 & 92.34 & 76.22 & 81.10 & 85.37 & 52.68 & 55.31 & 57.74 \\
DLE (min-p) & 90.52 & 91.43 & 91.74 & 76.83 & 84.76 & 89.02 & 54.34 & 56.13 & 58.23 \\
\midrule
Self-consistency ($\varepsilon$-sampling) & 89.23 & 91.36 & 92.72 & 73.17 & 83.54 & 88.41 & 52.94 & 56.48 & 58.56 \\
DLE ($\varepsilon$-sampling)-\textsc{DivFirst}  & 90.60 & 91.51 & 92.34 & 76.22 & 85.98 & 88.41 & 54.66 & 56.40 & 58.18 \\
DLE ($\varepsilon$-sampling)-\textsc{ProbFirst} & 90.07 & 91.43 & 92.19 & 76.83 & 85.98 & 89.02 & 54.33 & 56.38 & 58.30 \\
DLE ($\varepsilon$-sampling)-\textsc{RandBranch} & 89.92 & 91.51 & 92.42 & 78.66 & 84.76 & 88.41 & 54.88 & 56.52 & 58.29 \\
\bottomrule
\end{tabularx}
}
\caption{
Performance (maj@k and pass@k) of various methods on GSM8K, Humaneval and MMLU-Pro with the Qwen2.5-14B-Instruct model.
}
\end{table*}

\begin{table*}[!h]
\centering
\label{table:llama1b}
\setlength{\tabcolsep}{3pt}
\renewcommand{\arraystretch}{1.2}
\newcommand{\na}{\multicolumn{1}{c}{\textemdash}}
{\small
\begin{tabularx}{\textwidth}{lYYYYYYYYY}
\toprule
\multirow{2}{*}{Method}
& \multicolumn{3}{c}{GSM8K} & \multicolumn{3}{c}{Humaneval} & \multicolumn{3}{c}{MMLU-Pro} \\
\cmidrule(lr){2-4} \cmidrule(lr){5-7} \cmidrule(lr){8-10}
& {\small maj@2} & {\small maj@4} & {\small maj@8}
& {\small pass@2} & {\small pass@4} & {\small pass@8}
& {\small maj@2} & {\small maj@4} & {\small maj@8} \\
\midrule
Self-consistency & 18.50 & 24.64 & 32.60 & 25.00 & 37.20 & 46.95 & 13.73 & 16.71 & 19.02 \\
\midrule
Self-consistency (top-p+top-k) & 27.60 & 33.21 & 39.65 & 31.71 & 44.51 & 54.88 & 17.30 & 19.30 & 20.78 \\
DLE (top-p+top-k) & 33.06 & 36.01 & 38.06 & 42.07 & 53.66 & 59.15 & 20.62 & 22.17 & 22.86 \\
\midrule
Self-consistency (min-p) & 30.10 & 35.25 & 40.25 & 37.20 & 46.34 & 54.88 & 17.73 & 20.09 & 21.29 \\
DLE (min-p) & 33.21 & 36.01 & 38.44 & 42.07 & 53.05 & 59.76 & 20.71 & 22.17 & 22.85 \\
\midrule
Self-consistency ($\varepsilon$-sampling) & 27.75 & 33.89 & 39.95 & 37.80 & 48.78 & 54.88 & 18.87 & 20.69 & 21.82 \\
DLE ($\varepsilon$-sampling)-\textsc{DivFirst}  & 31.39 & 34.65 & 36.85 & 31.71 & 39.63 & 46.95 & 19.76 & 21.23 & 22.15 \\
DLE ($\varepsilon$-sampling)-\textsc{ProbFirst} & 33.21 & 36.01 & 38.06 & 41.46 & 52.44 & 59.76 & 19.61 & 21.13 & 22.95 \\
DLE ($\varepsilon$-sampling)-\textsc{RandBranch} & 33.13 & 34.80 & 38.44 & 43.29 & 48.17 & 57.32 & 19.75 & 20.90 & 21.97 \\
\bottomrule
\end{tabularx}
}
\caption{
Performance (maj@k and pass@k) of various methods on GSM8K, Humaneval and MMLU-Pro with the Llama3.2-1B-Instruct.
}
\end{table*}

\begin{table*}[!h]
\centering
\label{table:llama3b}
\setlength{\tabcolsep}{3pt}
\renewcommand{\arraystretch}{1.2}
\newcommand{\na}{\multicolumn{1}{c}{\textemdash}}
{\small
\begin{tabularx}{\textwidth}{lYYYYYYYYY}
\toprule
\multirow{2}{*}{Method}
& \multicolumn{3}{c}{GSM8K} & \multicolumn{3}{c}{Humaneval} & \multicolumn{3}{c}{MMLU-Pro} \\
\cmidrule(lr){2-4} \cmidrule(lr){5-7} \cmidrule(lr){8-10}
& {\small maj@2} & {\small maj@4} & {\small maj@8}
& {\small pass@2} & {\small pass@4} & {\small pass@8}
& {\small maj@2} & {\small maj@4} & {\small maj@8} \\
\midrule
Self-consistency & 43.29 & 54.36 & 65.58 & 46.34 & 57.93 & 70.12 & 24.88 & 29.31 & 33.04 \\
\midrule
Self-consistency (top-p+top-k) & 56.25 & 65.81 & 72.48 & 49.39 & 64.02 & 73.78 & 29.74 & 33.64 & 36.17 \\
DLE (top-p+top-k) & 64.52 & 69.45 & 71.49 & 59.15 & 67.07 & 76.83 & 33.22 & 35.73 & 37.28 \\
\midrule
Self-consistency (min-p) & 59.89 & 68.39 & 73.69 & 51.83 & 65.85 & 76.83 & 30.68 & 34.73 & 36.08 \\
DLE (min-p) & 64.44 & 68.84 & 71.57 & 59.15 & 68.29 & 76.83 & 33.23 & 35.71 & 36.91 \\
\midrule
Self-consistency ($\varepsilon$-sampling) & 58.38 & 62.33 & 69.14 & 55.49 & 64.46 & 75.61 & 30.98 & 35.37 & 37.36 \\
DLE ($\varepsilon$-sampling)-\textsc{DivFirst}  & 62.62 & 67.40 & 69.52 & 55.49 & 65.24 & 71.95 & 33.19 & 35.65 & 37.23 \\
DLE ($\varepsilon$-sampling)-\textsc{ProbFirst} & 64.44 & 68.92 & 71.65 & 59.15 & 69.51 & 78.66 & 33.18 & 35.88 & 37.53 \\
DLE ($\varepsilon$-sampling)-\textsc{RandBranch} & 64.52 & 68.31 & 72.33 & 57.93 & 65.24 & 73.17 & 33.22 & 35.60 & 37.49 \\
\bottomrule
\end{tabularx}
}
\caption{
Performance (maj@k and pass@k) of various methods on GSM8K, Humaneval and MMLU-Pro with the Llama3.2-3B-Instruct model.
}
\end{table*}


\clearpage

\subsection{Additional Results on Accuracy and Number of New Tokens Generated} \label[appendix]{app:sec:acc-number-tokens-generated}

As an extension to \figpartref{fig:accuracy-number-tokens-eps-ele}{left} which only shows DLE with $\varepsilon$-sampling for $k=32$, we additionally provide results for $k=8$. 
We draw the same conclusion that DLE uses fewer tokens for the same accuracy.

\begin{figure}[!h]
    \centering
    \includegraphics{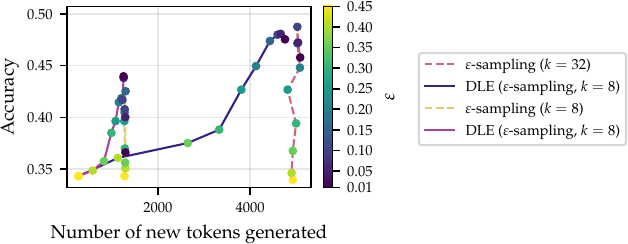}
    \caption{
    DLE with $\varepsilon$-sampling outperforms $\varepsilon$-sampling with far fewer new tokens generated.
    Number of new tokens is counted on average for each question.
    $\varepsilon$ is marked with colors indicating its value.
    Stochastic self-consistency generates many more tokens to achieve a similar performance as DLE.
    }
    \label{fig:accuracy-number-tokens-eps-ele-32}
\end{figure}

\subsection{Additional Results on Inference Speed} \label[appendix]{app:sec:computational-experiments}
In the main body, we discussed the inference speed benefit of using SGLang, which can output the cached token count for each generation request. However, vLLM does not output such data, so we cannot directly measure its cache retrieval rate. We can still see that there is some speed benefit from using DLE with $\varepsilon$-sampling at high $\varepsilon$ thresholds.

\begin{figure}[!h]
    \centering
    \includegraphics[width=\textwidth]{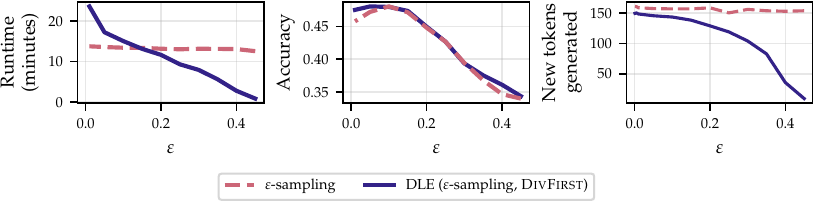}
    \caption{Inference time and performance for DLE with $\varepsilon$-sampling, measured using vLLM. At higher $\varepsilon$, inference time is shorter than the baseline. The performance boost is also more obvious since far fewer tokens are generated.}
    \label{fig:inference-time-performance}
\end{figure}

\clearpage
\subsection{Deep Think with Confidence (DeepConf) \citep{fu_deep_2025}} \label[appendix]{app:sec:deep-conf}

DeepConf \citep{fu_deep_2025} is a test-time method that filters out reasoning traces of low quality when doing inference or after inference.
We only investigate the offline version of DeepConf.

\citet{fu_deep_2025} defines the token confidence as the negative average $\log$-probability of the top-$n$ tokens at a given position $i$
\begin{align}
    C_i = -\frac{1}{n}\sum_{j=1}^n \log P^\mathrm{top}_i(j),
\end{align}
where $n$ denotes the number of tokens considered and $P^\mathrm{top}_i$ returns the probabilities of top tokens ordered decreasingly at position $i$.
The average trace confidence is simply defined as the mean over all token confidences.

Additionally, DeepConf introduces group confidence for groups of tokens, bottom $10\%$ group confidence, lowest group confidence, and tail confidence. 
For voting, DeepConf introduces confidence-weighted majority voting that weights each answer by its confidence, and confidence filtering that filters the top-$p\%$ of traces.

In our experiments, we set the group size to $64$ tokens (as opposed to $2048$ for the original work) as we are working with reasoning problems that require fewer tokens.
We experiment with $p\in\{10, 25\}\%$.

\Cref{table:deep-conf-gsm8k} shows results for all DeepConf experiments \citep{fu_deep_2025} on GSM8K \citep{cobbe_training_2021} with Qwen2.5-0.5B-Instruct.
\Cref{table:deep-conf-mmlu} shows results for all DeepConf experiments \citep{fu_deep_2025} on MMLU-Pro \citep{wang_mmlu-pro_2024} with Qwen2.5-0.5B-Instruct.
We include the best results achieved in \Cref{table:main}.

\begin{table*}[!h]
\setlength{\tabcolsep}{3pt}
\renewcommand{\arraystretch}{1.2}
\newcommand{\na}{\multicolumn{1}{c}{\textemdash}}
{\small
\begin{tabularx}{\textwidth}{lYYY}
\toprule
Method & {\small 2} & {\small 4} & {\small 8} \\
\addlinespace
\midrule
\multicolumn{4}{c}{\textbf{Mean Confidence}} \\
\midrule
\addlinespace
Weighted & 0.2002 & 0.2578 & 0.3298 \\
Top Percent Weighted (10\%) & 0.1948 & 0.2532 & 0.2616 \\
Top Percent Weighted (25\%) & 0.2153 & 0.2199 & 0.2570 \\
Top Percent (10\%) & 0.2115 & 0.2343 & 0.2509 \\
Top Percent (25\%) & 0.2055 & 0.2153 & 0.2464 \\
\addlinespace
\midrule
\multicolumn{4}{c}{\textbf{Tail Mean Conf (64 Tokens)}} \\
\midrule
\addlinespace
Weighted & 0.2085 & 0.2472 & 0.3222 \\
Top Percent Weighted (10\%) & 0.2002 & 0.2199 & 0.2350 \\
Top Percent Weighted (25\%) & 0.2032 & 0.2077 & 0.2456 \\
Top Percent (10\%) & 0.1941 & 0.2146 & 0.2358 \\
Top Percent (25\%) & 0.1903 & 0.2161 & 0.2525 \\
\addlinespace
\midrule
\multicolumn{4}{c}{\textbf{Tail Mean Conf (10\%)}} \\
\midrule
\addlinespace
Weighted & 0.1782 & 0.2252 & 0.3086 \\
Top Percent Weighted (10\%) & 0.1865 & 0.2077 & 0.1971 \\
Top Percent Weighted (25\%) & 0.1676 & 0.1956 & 0.2055 \\
Top Percent (10\%) & 0.1638 & 0.1850 & 0.2077 \\
Top Percent (25\%) & 0.1804 & 0.1933 & 0.2039 \\
\addlinespace
\midrule
\multicolumn{4}{c}{\textbf{Min Sliding Mean Conf (64 Tokens)}} \\
\midrule
\addlinespace
Weighted & 0.1895 & 0.2563 & 0.3283 \\
Top Percent Weighted (10\%) & 0.1873 & 0.2252 & 0.2464 \\
Top Percent Weighted (25\%) & 0.2092 & 0.2199 & 0.2418 \\
Top Percent (10\%) & 0.1842 & 0.2206 & 0.2593 \\
Top Percent (25\%) & 0.2077 & 0.2244 & 0.2214 \\
\addlinespace
\midrule
\multicolumn{4}{c}{\textbf{Bottom Sliding Mean Conf (10\%, 64 Tokens)}} \\
\midrule
\addlinespace
Weighted & 0.2077 & 0.2631 & \textbf{0.3336} \\
Top Percent Weighted (10\%) & 0.1979 & 0.2206 & 0.2661 \\
Top Percent Weighted (25\%) & 0.1774 & 0.2183 & 0.2487 \\
Top Percent (10\%) & 0.1979 & 0.2237 & 0.2449 \\
Top Percent (25\%) & 0.1676 & 0.2183 & 0.2441 \\
\addlinespace
\midrule
\multicolumn{4}{c}{\textbf{Bottom Sliding Mean Conf (50\%, 64 Tokens)}} \\
\midrule
\addlinespace
Weighted & 0.1895 & \textbf{0.2691} & 0.3139 \\
Top Percent Weighted (10\%) & 0.1918 & 0.2206 & 0.2578 \\
Top Percent Weighted (25\%) & \textbf{0.2168} & 0.2290 & 0.2638 \\
Top Percent (10\%) & 0.1850 & 0.2115 & 0.2623 \\
Top Percent (25\%) & \na & 0.2199 & 0.2631 \\
\bottomrule
\end{tabularx}
}
\caption{Results for DeepConf \citep{fu_deep_2025} on GSM8K \citep{cobbe_training_2021} with Qwen2.5-0.5B-Instruct. 
See \Cref{app:sec:deep-conf} for details.}
\label{table:deep-conf-gsm8k}
\end{table*}

\begin{table*}[!h]
\setlength{\tabcolsep}{3pt}
\renewcommand{\arraystretch}{1.2}
\newcommand{\na}{\multicolumn{1}{c}{\textemdash}}
{\small
\begin{tabularx}{\textwidth}{lYYY}
\toprule
Method & {\small 2} & {\small 4} & {\small 8} \\
\addlinespace
\midrule
\multicolumn{4}{c}{\textbf{Mean Confidence}} \\
\midrule
\addlinespace
Weighted & 0.1406 & \textbf{0.1558} & 0.1607 \\
Top Percent Weighted (10\%) & 0.1371 & 0.1451 & 0.1544 \\
Top Percent Weighted (25\%) & 0.1363 & 0.1524 & 0.1439 \\
Top Percent (10\%) & 0.1433 & 0.1493 & 0.1471 \\
Top Percent (25\%) & 0.1419 & 0.1439 & 0.1503 \\
\addlinespace
\midrule
\multicolumn{4}{c}{\textbf{Tail Mean Conf (64 Tokens)}} \\
\midrule
\addlinespace
Weighted & 0.1381 & 0.1518 & 0.1635 \\
Top Percent Weighted (10\%) & 0.1362 & \na & 0.1478 \\
Top Percent Weighted (25\%) & 0.1410 & 0.1468 & 0.1516 \\
Top Percent (10\%) & 0.1409 & 0.1435 & 0.1479 \\
Top Percent (25\%) & 0.1388 & 0.1464 & 0.1470 \\
\addlinespace
\midrule
\multicolumn{4}{c}{\textbf{Tail Mean Conf (10\%)}} \\
\midrule
\addlinespace
Weighted & 0.1410 & 0.1503 & 0.1600 \\
Top Percent Weighted (10\%) & 0.1377 & 0.1434 & 0.1410 \\
Top Percent Weighted (25\%) & 0.1385 & 0.1376 & 0.1446 \\
Top Percent (10\%) & 0.1341 & 0.1420 & 0.1497 \\
Top Percent (25\%) & 0.1394 & 0.1385 & 0.1444 \\
\addlinespace
\midrule
\multicolumn{4}{c}{\textbf{Min Sliding Mean Conf (64 Tokens)}} \\
\midrule
\addlinespace
Weighted & 0.1401 & 0.1497 & \textbf{0.1676} \\
Top Percent Weighted (10\%) & 0.1418 & 0.1470 & 0.1469 \\
Top Percent Weighted (25\%) & 0.1347 & 0.1465 & 0.1456 \\
Top Percent (10\%) & 0.1411 & 0.1478 & 0.1499 \\
Top Percent (25\%) & 0.1418 & 0.1457 & 0.1470 \\
\addlinespace
\midrule
\multicolumn{4}{c}{\textbf{Bottom Sliding Mean Conf (10\%, 64 Tokens)}} \\
\midrule
\addlinespace
Weighted & 0.1395 & 0.1547 & 0.1618 \\
Top Percent Weighted (10\%) & 0.1389 & 0.1432 & 0.1490 \\
Top Percent Weighted (25\%) & 0.1375 & 0.1463 & 0.1490 \\
Top Percent (10\%) & 0.1426 & 0.1462 & 0.1495 \\
Top Percent (25\%) & 0.1394 & 0.1418 & 0.1476 \\
\addlinespace
\midrule
\multicolumn{4}{c}{\textbf{Bottom Sliding Mean Conf (50\%, 64 Tokens)}} \\
\midrule
\addlinespace
Weighted & \textbf{0.1441} & 0.1501 & 0.1646 \\
Top Percent Weighted (10\%) & 0.1419 & 0.1523 & 0.1526 \\
Top Percent Weighted (25\%) & 0.1390 & 0.1467 & 0.1489 \\
Top Percent (10\%) & 0.1365 & 0.1439 & 0.1511 \\
Top Percent (25\%) & 0.1391 & 0.1462 & 0.1461 \\
\bottomrule
\end{tabularx}
}
\caption{Results for DeepConf \citep{fu_deep_2025} on MMLU-Pro \citep{wang_mmlu-pro_2024} with Qwen2.5-0.5B-Instruct.
See \Cref{app:sec:deep-conf} for additional information.}
\label{table:deep-conf-mmlu}
\end{table*}

\clearpage

\subsection{Self-Certainty \citep{kang_scalable_2025}} \label[appendix]{app:sec:self-certainty}

Self-certainty \citep{kang_scalable_2025} is a method for response evaluation and selection.
More specifically, it is a metric that takes the probability distribution of the LLM into account to estimate response quality without additional compute.
The general idea is that a larger self-certainty also corresponds to an improved response accuracy.
Self-certainty is defined as the point-wise Kullback-Leibler \citep{kullback1951information} divergence between a uniform distribution and the next-token distribution, averaged over the sequence length,
\begin{align}
    \mathrm{Self\text{-}Certainty}(\rvx_{1:T})
    \coloneqq
    -\frac{1}{VT} \sum_{t=1}^T\sum_{v=1}^V \log(V \cdot p(v\mid \rvx_{<t}))
    ,
\end{align}
where $V = |\mathcal V|$ is the size of the vocabulary.

\citet{kang_scalable_2025} also identifies that score-based voting methods suffer from sensitivity to score scaling. 
Therefore, they propose an approach inspired by Borda-counts.
First, they rank $N$ outputs of models by confidence, obtaining a ranking $(r_1, r_2, \dots, r_N)$.
Then, they assign votes to these ranked outputs using the formula
\begin{align}
    v(r) = (N-r+1)^p,
\end{align}
where $r$ is the rank of the output and $p$ is the voting power.

\Cref{tab:self-certainty-gsm8k} shows results for all self-certainty experiments on GSM8K \citep{cobbe_training_2021} with Qwen2.5-0.5B-Instruct.
\Cref{tab:self-certainty-mmlu} shows results for all self-certainty experiments on MMLU-Pro \citep{wang_mmlu-pro_2024} with Qwen2.5-0.5B-Instruct.

\begin{table*}[!ht]
\centering
\setlength{\tabcolsep}{3pt}
\renewcommand{\arraystretch}{1.2}
\newcommand{\na}{\multicolumn{1}{c}{\textemdash}}
{\small
\begin{tabularx}{\textwidth}{l YYYYYY}
\toprule
Voting power $p$ & {\small $\argmax, p=0.0$} & {\small $p=0.1$} & {\small $p=0.3$} & {\small $p=0.5$} & {\small $p=0.7$} & {\small $p=0.9$} \\
\midrule
Voting $k=2$ & \textbf{20.09} & 19.94 & 19.94 & 20.02 & 19.94 & 19.94 \\
Voting $k=4$ & 23.81 & 25.09 & 24.87 & 24.41 & 24.64 & \textbf{25.32} \\
Voting $k=8$ & 28.73 & 29.80 & \textbf{32.22} & 31.01 & 30.17 & 31.54 \\
\bottomrule
\end{tabularx}
}
\caption{Results for self-certainty \citep{kang_scalable_2025} on GSM8K \citep{cobbe_training_2021} with Qwen2.5-0.5B-Instruct.
$\argmax, p=0.0$ denotes the case \textit{without voting} and choosing the $\argmax$ sequence for self-certainty.}
\label{tab:self-certainty-gsm8k}
\end{table*}

\begin{table*}[!ht]
\centering
\setlength{\tabcolsep}{3pt}
\renewcommand{\arraystretch}{1.2}
\newcommand{\na}{\multicolumn{1}{c}{\textemdash}}
{\small
\begin{tabularx}{\textwidth}{l YYYYYY}
\toprule
Voting power $p$ & {$\argmax, p=0.0$} & {\small $p=0.1$} & {\small $p=0.3$} & {\small $p=0.5$} & {\small $p=0.7$} & {\small $p=0.9$} \\
\midrule
Voting $k=2$ & 14.42 & 14.26 & 14.21 & 14.11 & 14.45 & \textbf{14.51} \\
Voting $k=4$ & 15.13 & 15.40 & \textbf{15.92} & 15.90 & 15.10 & 15.91 \\
Voting $k=8$ & 13.48 & 14.25 & 14.66 & \textbf{14.68} & 14.28 & 14.14 \\
\bottomrule
\end{tabularx}
}
\caption{Results for self-certainty \citep{kang_scalable_2025} on  MMLU-Pro \citep{wang_mmlu-pro_2024} with Qwen2.5-0.5B-Instruct.
$\argmax, p=0.0$ denotes the case \textit{without voting} and choosing the $\argmax$ sequence for self-certainty.}
\label{tab:self-certainty-mmlu}
\end{table*}

\clearpage


\end{document}